%% file: preprint.tex
\documentclass{article}
\usepackage{iclr2027_conference,times}
\input{macros}
\title{Trapped by Their Own Rollouts:\\
Understanding Aggregation--Rollout Feedback in 
Federated On-Policy Distillation}
\iclrfinalcopy
\author{
Jinqian Chen, Jihua Zhu\thanks{Corresponding author}, Chang Liu \\
School of Software Engineering, Xi'an Jiaotong University \\
\texttt{chenjinqian@stu.xjtu.edu.cn, zhujh@xjtu.edu.cn} 
}

\begin{document}
\maketitle
\lhead{Preprint. Under review.}
\input{sections/abstract}
\input{sections/introduction}
\input{sections/theory}
\input{sections/method}
\input{sections/experiments}
\FloatBarrier
\input{sections/related_work}
\input{sections/conclusion}
\label{maintext:end}
\clearpage
\input{sections/statements}
\bibliography{references}
\bibliographystyle{iclr2027_conference}
\clearpage
\appendix
\input{appendices/notation}

\input{appendices/general_proofs}
\input{appendices/construction_proofs}
\input{appendices/intervention_proofs}
\input{appendices/algorithm_details}
\input{appendices/experimental_protocol}
\input{appendices/additional_experiments}

\input{appendices/main_results}

\input{appendices/scaleup_results}

\end{document}

%% file: macros.tex
\usepackage[T1]{fontenc}
\usepackage{amsmath,amssymb,amsthm,mathtools,bm}
\usepackage{graphicx,booktabs,array,multirow,tabularx,longtable}
\usepackage[table]{xcolor}
\usepackage{wrapfig,adjustbox,etoolbox}
\usepackage[font=small,skip=4pt]{caption}
\usepackage{algorithm,algpseudocode}
\usepackage{placeins}
\usepackage{tikz}
\usetikzlibrary{arrows.meta,positioning,calc}
\usepackage{microtype}
\usepackage{hyperref}
\usepackage{url}
\hypersetup{
  colorlinks=false,
  pdfborder={0 0 1},
  citebordercolor={0 1 0},
  linkbordercolor={0 1 0},
  pdftitle={Trapped by Their Own Rollouts},
  pdfauthor={Anonymous authors}
}
\newcommand{\method}{FedTOPS}

\newcommand{\E}{\mathbb{E}}

\newcommand{\KL}{D_{\mathrm{KL}}}
\newcommand{\TV}{\operatorname{TV}}
\newcommand{\Var}{\operatorname{Var}}
\newcommand{\ind}{\mathbf{1}}
\newcommand{\F}{\mathcal{F}_{\eta}}
\newcommand{\Pmap}{\mathcal{P}_{\eta}}
\newcommand{\cA}{\mathcal{A}}
\newcommand{\cC}{\mathcal{C}}
\newcommand{\cS}{\mathcal{S}}
\newcommand{\norm}[1]{\left\lVert#1\right\rVert}
\newcommand{\pp}{\,\mathrm{pp}}
\DeclareMathOperator*{\argmin}{arg\,min}
\newtheorem{theorem}{Theorem}
\newtheorem{proposition}[theorem]{Proposition}
\newtheorem{lemma}[theorem]{Lemma}
\newtheorem{corollary}[theorem]{Corollary}
\theoremstyle{definition}
\newtheorem{assumption}{Assumption}

\theoremstyle{remark}

\newcolumntype{Y}{>{\raggedright\arraybackslash}X}
\newcommand{\benchheader}{\rowcolor{headbg}\textbf{Method} & \textbf{MATH} & \textbf{AMC} & \textbf{Minerva} & \textbf{Olympiad} & \textbf{AIME24} & \textbf{AIME25} & \textbf{Macro}\\}

\newcommand{\seqsplit}[1]{\nolinkurl{#1}}

\definecolor{headbg}{HTML}{E8EDF2}
\definecolor{oursbg}{HTML}{E6F2F1}
\AtBeginEnvironment{table}{\small\renewcommand{\arraystretch}{1.06}\setlength{\tabcolsep}{4pt}}

\AtBeginDocument{%
  \setlength{\abovedisplayskip}{7pt plus 2pt minus 2pt}%
  \setlength{\belowdisplayskip}{7pt plus 2pt minus 2pt}%
  \setlength{\abovedisplayshortskip}{3pt plus 1pt}%
  \setlength{\belowdisplayshortskip}{5pt plus 1pt minus 1pt}}
\makeatletter
\renewcommand{\paragraph}{\@startsection{paragraph}{4}{\z@}{1.1ex plus .3ex minus .2ex}{-1em}{\normalsize\bfseries}}
\makeatother

%% file: sections/abstract.tex
\begin{abstract}
On-policy distillation (OPD) is a promising approach to language-model adaptation, aligning teacher supervision with the student's own generated trajectories. When adaptation prompts are distributed across clients, can this process benefit from federated collaboration? We study federated OPD and find that substantial collaboration gains can be obscured by learning-rate sensitivity: FedAvg can perform no better than independent local training at a small learning rate, yet recover a clear advantage at a larger rate. We explain this phenomenon through the student's dual role as learner and generator of future training data. An aggregation-induced optimization lag can delay access to useful teacher supervision, which in turn slows subsequent learning. Our theory establishes this \emph{aggregation--rollout feedback} in a solvable model with a common optimum and stable updates, and identifies two coupled roles of learning rate: learning from current supervision and reaching future supervision. Guided by this analysis, we propose \method\ (Federated Teacher-guided On-Policy Scaling), which reuses teacher feedback on current trajectories to adapt the FedAvg update magnitude under clientwise predictive-change constraints. Across six mathematical reasoning benchmarks, \method\ improves macro Avg@8 over FedAvg by 4.56--14.57 percentage points across the evaluated student models and local learning rates.
\end{abstract}

%% file: sections/introduction.tex
\section{Introduction}
\label{sec:intro}
On-policy distillation (OPD) offers an effective way to adapt language models using a stronger teacher \citep{agarwal2024gkd,gu2024minillm}. By learning from teacher feedback on its own generations, a student receives supervision on the states it actually visits. In downstream applications, however, the prompts needed for adaptation may be held by separate clients and cannot be pooled. Federated learning offers a natural way for these clients to collaborate by sharing model updates. This raises a basic question: \emph{can on-policy distillation benefit from federated collaboration?}

At first glance, this appears to be a familiar federated optimization problem: combine local updates while accounting for heterogeneous client objectives \citep{mcmahan2017fedavg,karimireddy2020scaffold}. Yet OPD gives the student a second role. It is both \emph{the model being trained} and \emph{the generator of future training data}. In conventional fixed-data federated learning, aggregation changes the learner while the available examples remain fixed. In FedOPD, it also changes the trajectories on which the teacher will provide subsequent supervision. Even when each client's prompt distribution stays fixed, the distribution of visited prefixes evolves with the student. Aggregation therefore couples progress across heterogeneous clients with changes in what those clients will learn from next.

We present, to our knowledge, the first systematic study of FedOPD under a shared teacher. We uncover a counterintuitive pattern: federation can provide essentially no gain over independent local training at a small learning rate, yet become substantially beneficial at a larger one (Figure~\ref{fig:motivation}). Learning rate therefore governs more than the pace of optimization: \textbf{it can determine whether OPD benefits from federated collaboration within a given training budget}. This sensitivity raises a question specific to the student's dual role: how does the progress made by aggregation affect the supervision available in subsequent rounds?

Our central insight is that \textbf{an aggregation-induced optimization lag can become a lag in access to future supervision}. If the aggregated student progresses more slowly toward an ability needed to reach useful teaching states, those states occur less often in its next rollouts. The student then receives less of the supervision that would support further progress. This creates \emph{aggregation--rollout feedback}: an optimization discrepancy affects the data generated next, which in turn affects subsequent learning. A shared, fixed teacher does not eliminate this feedback, because access to its supervision still depends on the student.

\begin{wrapfigure}[]{r}{0.56\textwidth}
\centering
\includegraphics[width=\linewidth]{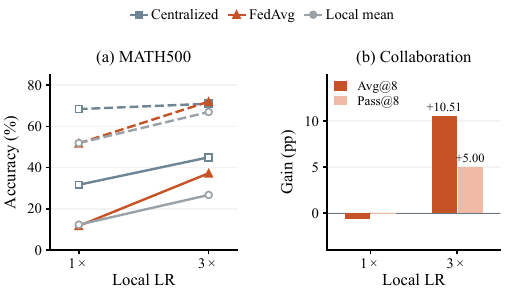}
\caption{\textbf{Learning rate changes collaboration gains.} MATH500, student: Qwen3-0.6B, teacher: Qwen3-8B, matched total rollout budgets. Left: Avg@8 (solid) and Pass@8 (dashed). Right: FedAvg minus Local mean.}
\label{fig:motivation}
\end{wrapfigure}

Our theory makes this mechanism precise. In an exactly solvable two-stage model, aggregation delays access to downstream teaching states even though all clients share a realizable teacher and a common optimum, and every local update remains stable. Changing only the rollout source isolates the resulting supervision delay. The analysis reveals two coupled roles of learning rate: \emph{learning from current supervision} and \emph{reaching future supervision}. Their interaction reproduces the observed change in the value of collaboration within a finite budget. A general update decomposition further separates ordinary optimization effects from the propagation of discrepancies through the evolving rollout distribution.

This understanding suggests calibrating the aggregate before it generates the next round's trajectories. We propose \method\ (Federated Teacher-guided On-Policy Scaling), which uses teacher feedback on current training trajectories to select how far to move along the FedAvg update. Clientwise predictive-change constraints limit the adjustment, and reusing the trajectories avoids additional rollout generation. A controlled server-scaling experiment supports update magnitude as a useful intervention. Across six mathematical reasoning benchmarks, \method\ improves macro Avg@8 over FedAvg by 4.56--14.57 percentage points across the evaluated student models and learning rates.

Our contributions are threefold:
\begin{itemize}
\item We establish FedOPD as a setting where collaboration can improve distillation, while revealing that its benefits depend strongly on the local learning rate.
\item We explain how aggregation-induced optimization lag can delay future supervision and characterize the two roles of learning rate through a solvable construction and a general feedback decomposition.
\item We introduce \method, a simple teacher-guided aggregation rule, and demonstrate improved performance across student models and local learning rates.
\end{itemize}

%% file: sections/theory.tex
\section{Phenomenon and Theoretical Analysis}
\label{sec:theory}
\label{sec:setup}
We first establish the learning-rate sensitivity of collaboration gains, then explain it through the student's coupled roles as learner and generator.

\subsection{When Does Collaboration Help?}
Figure~\ref{fig:motivation} compares FedAvg, independent Local training, and Centralized training from a common 0.6B student under matched total rollout budgets. Local mean averages three separately trained client models. At the reference learning rate, FedAvg offers essentially no collaboration gain; at the larger rate, its MATH500 Avg@8 gain over Local rises from $-0.60$ to $+10.51\pp$. The six-benchmark macro average follows the same pattern.

The key observation is that learning rate changes whether collaboration is useful within the budget. In OPD, aggregation also sets the generator of the next training trajectories. An optimization discrepancy can therefore change the supervision available in subsequent rounds.

\paragraph{FedOPD protocol.}
Clients have prompt distributions $D_i$ and positive weights $w_i$ summing to one, and share a fixed teacher $q$.
Each round starts from $\theta_t$: client $i$ generates student trajectories and minimizes
\begin{equation}
\mathcal{L}_i(\theta)
=
\frac{1}{|\mathcal{H}_i|}
\sum_{h\in\mathcal{H}_i}
D_{\mathrm{KL}}\!\left(p_\theta(\cdot\mid h)\,\|\,q(\cdot\mid h)\right),
\end{equation}
where $\mathcal{H}_i$ contains sampled response-token prefixes, held fixed during optimization.
The server aggregates the updated parameters as
$\theta_{t+1}=\sum_i w_i\theta_{t,i}$.
Independent Local training uses the same objective but retains each client's learner and generator without aggregation.

\paragraph{One model, two roles.}
Distinguish the learner $\pi_\theta$ from the generator $\pi_\phi$. Let $d_i^\phi$ be its prefix distribution, averaging positions up to horizon $L$ with zero-loss padding after termination (Appendix~\ref{app:general}). The fixed-source loss and training gradient are
\begin{equation}
L_i(\theta;\phi)=\E_{s\sim d_i^\phi}[\ell(\theta;s)],
\qquad g_i(\theta;\phi)=\nabla_\theta L_i(\theta;\phi),
\label{eq:twoarg}
\end{equation}
where $\ell$ is a conditional distillation loss, such as forward or reverse token KL. On-policy rounds start at $\phi=\theta$, but the training field $g_i^\circ(\theta)=g_i(\theta;\theta)$ does not differentiate through the sampled-prefix distribution.

\paragraph{Theoretical protocol.}
We analyze $E$ population gradient steps of size $\eta$, freezing the source within each round. Let $\F(\theta;\phi)$ denote local steps followed by averaging, and $\Pmap(\theta;\phi)$ denote $E$ synchronous steps on $\sum_iw_iL_i(\theta;\phi)$. This pooled reference isolates the effect of local optimization followed by aggregation. Appendix~\ref{app:general} states the regularity assumptions.

\subsection{How Optimization Lag Delays Supervision}
\label{sec:construction}
Consider a prerequisite ability $u$ that controls access to teaching states and a downstream ability $v$ learned there. A two-stage Gaussian autoregressive model realizes this setting exactly (Appendix~\ref{app:construction}). For client features $x_i>0$, let $h_i=x_i^2$, the access threshold be $c<b$, and the precision be $\lambda>0$. The access probability is
\begin{equation}
\rho_i(u)=\Phi\big(x_i(u-c)\big),
\label{eq:rho}
\end{equation}
where $\Phi$ is the standard normal CDF. Summing the two conditional distillation losses gives exactly
\begin{equation}
L_i((u,v);\phi)
=\tfrac{h_i}{2}(u-b)^2
+\tfrac{\lambda\rho_i(\phi_u)}{2}(v-1)^2.
\label{eq:construction_loss}
\end{equation}
Forward and reverse conditional KL agree here. All clients share a realizable teacher and optimum $(b,1)$, and every fixed-source objective is a convex quadratic. Thus $u$ controls access to supervision for $v$, without conflicting client targets.

Let $u_0<c<b$, $0\le v_0<1$, and $0<\eta\max_i\{h_i,\lambda\}<1$. Define
\begin{equation}
 r_\eta=\sum_iw_i(1-\eta h_i)^E,\qquad
 s_\eta(u)=\sum_iw_i(1-\eta\lambda\rho_i(u))^E.
\label{eq:rs}
\end{equation}
Exact local updates followed by averaging satisfy
\begin{equation}
 b-u_{t+1}^F=r_\eta(b-u_t^F),\qquad
 1-v_{t+1}^F=s_\eta(u_t^F)(1-v_t^F).
\label{eq:exact_recursion}
\end{equation}
For synchronous pooled updates, the corresponding factors are
$r_\eta^P=(1-\eta\bar h)^E$ and
$s_\eta^P(u)=(1-\eta\lambda\bar\rho(u))^E$, with weighted means $\bar h$ and $\bar\rho$.

\begin{theorem}[Stable optimization can delay future teaching]
\label{thm:bottleneck}
For the construction above and $E\ge2$, $u_t^F\le u_t^P$ and $v_t^F\le v_t^P$ at every round. Define a source-swapped process that retains federated learner updates but generates each round with $u_t^P$ instead of $u_t^F$. Then
\begin{equation}
 v_T^F\le v_T^{\mathrm{swap}}\le v_T^P,
\qquad
\log\frac{1-v_T^F}{1-v_T^{\mathrm{swap}}}
=\sum_{t=0}^{T-1}\int_{u_t^F}^{u_t^P}\kappa_\eta(z)\,\mathrm dz,
\label{eq:mediation}
\end{equation}
where $\kappa_\eta(u)=-\partial_u\log s_\eta(u)>0$. If each $\rho_i$ is made model-independent, the source-swap difference is exactly zero. Under the original Gaussian access law, FedAvg, pooled, and independent Local iterates converge monotonically to the common optimum.
\end{theorem}

\paragraph{Why this is an additional OPD cost.}
Jensen's inequality gives $r_\eta\ge r_\eta^P$: averaging delays prerequisite learning, reducing $\rho_i$ and hence downstream supervision. For $E=2$, the one-round prerequisite lag is exactly $\eta^2\Var_w(h_i)(b-u)$. Changing only the source in Eq.\eqref{eq:mediation} isolates the additional rollout-mediated cost. This is a finite-budget bottleneck under stable updates, even though all processes converge to the same target.

\subsection{The Two Roles of Learning Rate}
\label{sec:clock}
Learning rate affects both movement toward the teacher and the states available for subsequent learning. The residual downstream error makes these two roles explicit:
\begin{equation}
1-v_T^F=(1-v_0)e^{-\cC_T(\eta)},\qquad
\cC_T(\eta)=-\sum_{t=0}^{T-1}\log s_\eta(u_t^F(\eta)).
\label{eq:clock}
\end{equation}
The quantity $\cC_T$ is an \emph{effective supervision clock}: it depends on both the update size and the history of visited teaching states. Its derivative separates these effects:
\begin{equation}
\frac{\mathrm d\cC_T}{\mathrm d\eta}
=\underbrace{-\sum_t\partial_\eta\log s_\eta(u_t^F)}_{\text{learning on the same prefixes}}
+\underbrace{\sum_t\kappa_\eta(u_t^F)\frac{\mathrm du_t^F}{\mathrm d\eta}}_{\text{accessing different supervision}}.
\label{eq:clock_derivative}
\end{equation}
Both terms are nonnegative in the stable construction; the second disappears under a common source schedule independent of $\eta$. The rate also determines whether substantial supervision is reached within the training budget; Appendix~\ref{app:clock} gives the exact access-time formula.

\begin{wrapfigure}{r}{0.52\textwidth}
\centering
\includegraphics[width=\linewidth]{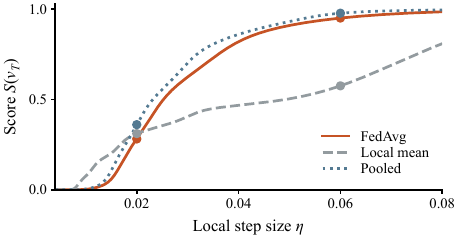}
\caption{\textbf{A sign reversal with stable updates.} Exact recurrences; marked step sizes 0.02 and 0.06. Details: Appendix~\ref{app:numerics}.}
\label{fig:toy}
\end{wrapfigure}
\paragraph{An explicit sign reversal against Local.}
Figure~\ref{fig:toy} instantiates the exact recurrences with two heterogeneous clients and evaluates downstream learning through $S(v)=1-(1-v)^2$ on a common fixed teaching state. At the smaller step size, FedAvg scores below the mean of independently trained Local models; tripling the step size reverses this ordering. All updates remain stable. Swapping the source alone recovers most of the gap to pooled training at both step sizes. This reproduces the change in collaboration gain and isolates a contribution from delayed supervision. Appendix~\ref{app:numerics} gives parameters and scores.

\paragraph{What changes when supervision is fixed?}
Replace $\rho_i(u)$ by client-specific constants while retaining the protocol and score $S(v)$. Corollary~\ref{cor:static_local} gives a nonnegative FedAvg-minus-Local gain at every stable step size. Thus fixed curvature heterogeneity alone cannot explain this construction's negative gain: model-dependent supervision access is essential.

\subsection{Separating the Two Effects in General}
For a general OPD update, let $H_i=\partial_\theta g_i$ and $B_i=\partial_\phi g_i$. The on-policy derivative is $Dg_i^\circ=H_i+B_i$: $H_i$ describes learning at a fixed source, while $B_i$ couples teacher gradients to earlier generation decisions. The latter has no predetermined sign. Appendix~\ref{app:general} derives its score-function expression.

\begin{proposition}[Creation and propagation of a discrepancy]
\label{prop:decomposition}
Under the regularity conditions of Appendix~\ref{app:general}, for fixed $E$ and a frozen source $\phi$,
\begin{equation}
\F(\theta;\phi)-\Pmap(\theta;\phi)
=\binom E2\eta^2\sum_iw_i(H_i-H)(g_i-g)+O(\eta^3),
\label{eq:local_gap}
\end{equation}
where $g_i,H_i$ are evaluated at $(\theta;\phi)$, $g=\sum_iw_ig_i$, and $H=\sum_iw_iH_i$. For the self-generated processes, write $\delta_t=\theta_t^F-\theta_t^P$ and $e_t=(\F-\Pmap)(\theta_t^P;\theta_t^P)$. Then locally
\begin{equation}
\delta_{t+1}
=\big[I-\eta E(H_t+B_t)\big]\delta_t+e_t
+O\!\left(\eta^2\norm{\delta_t}+\norm{\delta_t}^2\right).
\label{eq:propagation}
\end{equation}
The matrices in~Eq.\eqref{eq:propagation} are evaluated at the synchronous reference, and $B_t=\sum_iw_iB_{i,t}$.
\end{proposition}

Eq.\eqref{eq:local_gap} creates an ordinary aggregation discrepancy, which vanishes at $E=1$. Eq.\eqref{eq:propagation} propagates it through both learning and generation. A common external source removes the first-order $B_t\delta_t$ contribution, motivating the source intervention in Section~\ref{sec:mechanism_experiments}.

\paragraph{From explanation to intervention.}
The aggregate determines both the next learner and its training trajectories. We therefore use current teacher feedback to calibrate the aggregate update magnitude, addressing the coupled roles of learning rate without changing the update direction.

%% file: sections/method.tex
\section{Teacher-Guided Magnitude Selection}
\label{sec:method}
The analysis suggests a practical design principle: calibrate the aggregated update before it becomes the source of the next training trajectories. Moving farther along a useful update can improve the current student and advance its access to future supervision. Whether a larger move remains useful, however, depends on the current model and local updates. \method\ (Federated Teacher-guided On-Policy Scaling) therefore uses the teacher to select the update magnitude in each round.

\subsection{From the Analysis to an Aggregation Rule}
Let $\bar\theta_t$ be the FedAvg result and $\Delta_t=\bar\theta_t-\theta_t$ its increment from the round's starting model. The two-stage construction gives a concrete justification for scaling this increment. A constant multiplier satisfying
\begin{equation}
1\le a\le\min\!\left\{\frac1{1-r_\eta},\frac1{1-(1-\eta\lambda)^E}\right\},
\label{eq:safe_scaling}
\end{equation}
keeps updates non-overshooting and both coordinates no worse than FedAvg. Moreover, the frozen-source teacher loss decreases along the safe interval (Appendix~\ref{app:intervention}). Thus teacher feedback on the current trajectories can help select an update that also advances future learning.

\method\ implements this principle with a small set of candidates:
\begin{equation}
\theta_t(a)=\theta_t+a\Delta_t,
\qquad a\in\cA=\{0.5,1,2,3,5,10\}.
\label{eq:candidates}
\end{equation}
Here $\theta$ contains the trainable LoRA coordinates. The local learning rate $\eta$ determines the client updates; the server multiplier $a$ scales their aggregate, with $a=1$ recovering FedAvg. All candidates follow the same increment.

\subsection{Selecting an Update from Current Teacher Feedback}
\paragraph{Reuse the training trajectories.}
Before local optimization, each client caches a small sample of the round's student-generated prefixes and their teacher distributions. Every candidate is evaluated on this same cache. This gives a common basis for comparison without generating additional trajectories or querying the teacher for each candidate.

Let $\widetilde q$, $\widetilde\pi_t$, and $\widetilde\pi_{t,a}$ denote the cached teacher, starting-student, and candidate distributions, respectively. The tilde denotes a shared vocabulary partition, specified below. We measure teacher agreement and predictive change by
\begin{equation}
\widehat J_t(a)=\sum_iw_i\widehat\E_i
 \KL(\widetilde q\|\widetilde\pi_{t,a}),
\qquad
\widehat K_{i,t}(a)=\widehat\E_i
 \KL(\widetilde\pi_t\|\widetilde\pi_{t,a})\le\delta\quad\forall i.
\label{eq:selector}
\end{equation}
Here $\widehat\E_i$ averages cached positions on client $i$, and $\delta$ is a predictive-change budget shared across clients within a run. The forward KL in $\widehat J_t$ favors candidates closer to the teacher. The constraint is imposed separately on each client, so a small weighted-average change cannot conceal a large change on one client's trajectories.

\paragraph{Choose the magnitude.}
Among feasible candidates, \method\ selects the smallest teacher score. FedAvg is the default when feasible and is replaced only by a strictly better score. If no candidate satisfies the constraint, the server retains the round's starting model. Algorithm~\ref{alg:fedtops} in Appendix~\ref{app:algorithm} summarizes one round and the implementation details.

\paragraph{A compact scoring cache.}
Each client caches four training rollouts and at most 16 evenly spaced prediction positions per rollout. The vocabulary partition is the teacher/starting-student top-16 union plus the remaining probability mass, fixed across candidates. Cache details and the full-distribution analysis appear in Appendices~\ref{app:algorithm} and~\ref{app:intervention}.

\subsection{Connection to Learning-Rate Robustness}
Magnitude selection can offset local learning-rate changes when these mainly rescale the update. For $\Delta(\eta)=\eta d+O(\eta^2)$, write $\beta=a\eta$. The local quadratic teacher-score problem under the predictive-change constraint yields a preferred effective step $\beta^\star$, and hence $a^\star\simeq\beta^\star/\eta$ (Appendix~\ref{app:intervention}). This compensates update scale while preserving the FedAvg direction.

%% file: sections/experiments.tex
\section{Experiments}
\label{sec:experiments}
We evaluate collaboration gains across learning rates and student scales (RQ1), adaptive versus fixed aggregate scaling (RQ2), and rollout-source effects on subsequent learning (RQ3).

\input{tables/main_avg}
\input{tables/main_pass}

\subsection{Experimental Setup}
\label{sec:experimental_setup}
\paragraph{Models and federated task.}
We distill Qwen3-8B into Qwen3-0.6B-Base and Qwen3-1.7B-Base students \citep{yang2025qwen3}. Three clients each hold 570 MATH prompts \citep{hendrycks2021math} from disjoint subject groups: foundational algebra, discrete mathematics, and advanced algebra. All clients participate in 50 rounds, generating 32 responses each per round. We aggregate rank-8 LoRA factor coordinates. The 0.6B local learning rates are $3\times10^{-6}$ (1$\times$LR) and $9\times10^{-6}$ (3$\times$LR); 1.7B uses the reference rate. Each configuration uses one training seed.

\paragraph{Comparisons and evaluation.}
Local trains clients independently; Centralized pools their prompts. FedAvg averages updates; fixed $a=3$/$a=10$ amplify this increment every round without selection or a trust constraint. \method\ selects its magnitude (Algorithm~\ref{alg:fedtops}), with trust budgets 0.05 for all student models. All trained methods share initialization. Each Local model uses 1,600 rollouts; each federated or Centralized run uses 4,800 system-wide. The adapter-free 0.6B Base is evaluated once.
We sample eight answers per problem on MATH500, AMC23, Minerva, OlympiadBench, and AIME24/25 \citep{hendrycks2021math,lewkowycz2022minerva,he2024olympiadbench}. Avg@8 measures average correctness; Pass@8 measures whether any answer is correct. Macro weights benchmarks equally; Local mean averages client-model scores. Gains are percentage points (pp). Appendices~\ref{app:protocol} and~\ref{app:mainresults} give settings and client-level results.

\subsection{Recovering Collaboration Across Learning Rates and Student Scales}
\label{sec:mainresults}
\paragraph{Weak FedAvg performance conceals substantial collaboration gains.}
At the reference learning rate, FedAvg reaches only 3.91\% macro Avg@8, close to the untrained student's 3.32\% and below Local's 4.13\% (Table~\ref{tab:main}). With the same local learning rate, \method\ reaches 16.84\%, improving on FedAvg by $12.93\pp$ and exceeding the Centralized reference. This recovery extends across all six benchmarks. Pass@8 improves from 19.17\% to 34.69\% (Table~\ref{tab:main_pass}), showing gains in both average correctness and the probability of finding a correct response. Weak FedAvg performance need not imply little potential for collaboration.

\paragraph{The advantage persists at the larger learning rate.}
At 3$\times$LR, FedAvg itself gains from collaboration, but \method\ further improves macro Avg@8 by $4.56\pp$ and Pass@8 by $3.88\pp$. Its macro Avg@8 is similar at the two evaluated settings: 16.84\% and 17.33\%, compared with 3.91\% and 12.78\% for FedAvg. The $0.49\pp$ versus $8.87\pp$ gap describes these two configurations. Table~\ref{tab:main} shows positive collaboration gains on every benchmark at both rates, although \method\ ties Centralized on Minerva and trails it on AIME25 at the larger rate.

\paragraph{The benefit extends to a larger student.}
With the 1.7B student, \method\ improves macro Avg@8 over FedAvg by $14.57\pp$, reaching the Centralized reference, and achieves the highest Pass@8 among the compared methods (Tables~\ref{tab:main} and~\ref{tab:main_pass}). Its cap-hit rate falls from FedAvg's 38.66\% to 9.89\%, accompanying the accuracy improvement (Appendix~\ref{app:scaleup}). This experiment tests transfer across student scales at the reference learning rate.

\subsection{Aggregate Scaling and Adaptive Selection}
\label{sec:ablation}
\paragraph{Fixed amplification recovers much of the gain.}
Tables~\ref{tab:main} and~\ref{tab:main_pass} include fixed $a=3$ and $a=10$ at both 0.6B learning rates. At 1$\times$LR, increasing the server multiplier raises macro Avg@8 from FedAvg's 3.91\% to 10.81\% and 16.19\%, respectively, without changing local optimization. Thus a simple magnitude correction already recovers much of the collaboration gain.

\paragraph{The effective multiplier changes with the training setting.}
The preference between fixed rules reverses at 3$\times$LR (Figure~\ref{fig:fixed_scaling}): $a=3$ reaches 16.40\% macro Avg@8, whereas $a=10$ falls to 11.32\%, below FedAvg's 12.78\%. Always taking the largest candidate is therefore insufficient. \method\ achieves the highest macro Avg@8 in both settings, at 16.84\% and 17.33\%. It does not dominate every metric: fixed $a=10$ has higher macro Pass@8 at 1$\times$LR (34.99\% versus 34.69\%), and fixed $a=3$ does at 3$\times$LR (35.96\% versus 35.50\%). These controls motivate selecting the magnitude from current teacher feedback.

\begin{figure}[!t]
\centering
\includegraphics[width=\linewidth]{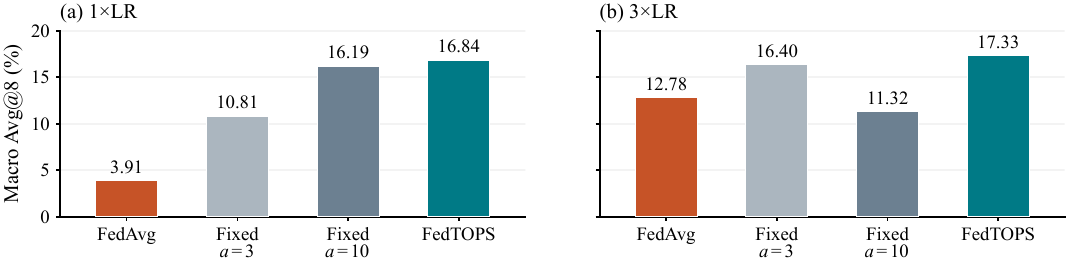}
\caption{\textbf{Fixed amplification versus adaptive selection.} Six-benchmark macro Avg@8 for the 0.6B student. The stronger fixed rule changes from $a=10$ to $a=3$ as the local learning rate increases; \method\ performs well in both settings. Tables~\ref{tab:main} and~\ref{tab:main_pass} give benchmark-level accuracy.}
\label{fig:fixed_scaling}
\end{figure}

\begin{figure}[!t]
\centering
\includegraphics[width=\linewidth]{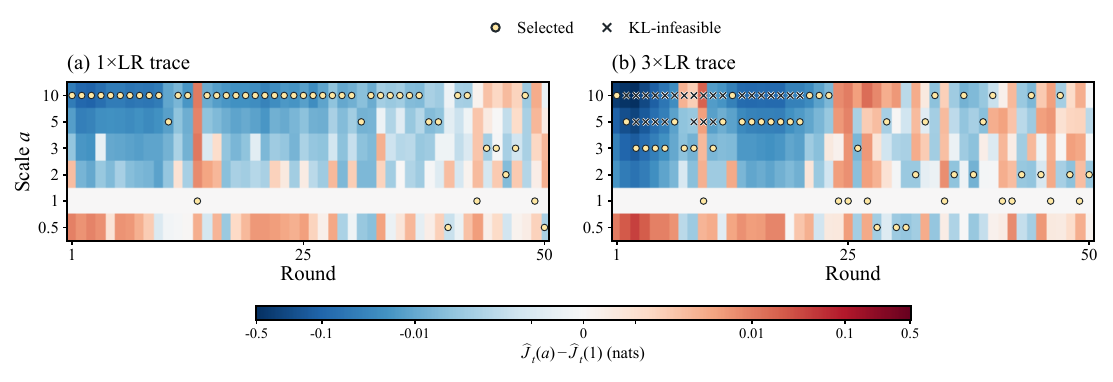}
\caption{\textbf{Teacher-guided selection across rounds.} Candidate-score differences $\widehat J_t(a)-\widehat J_t(1)$ on a shared symmetric-log scale: blue is better, circles are selected, crosses are infeasible. }
\label{fig:selection}
\end{figure}

\paragraph{Selection changes across rounds.}
The reference-rate trace selects $a=10$ in 37 of 50 rounds, but also selects every other candidate at least once. The higher-rate trace uses $a=10$ in 10 rounds and smaller scales in the remaining 40. Within each run, Figure~\ref{fig:selection} shows how teacher scores and feasibility vary as training evolves. Frequent upper-bound selections coexist with rounds that favor smaller steps, consistent with the need to reassess magnitude during training.

\paragraph{Computational cost.}
\method\ reuses training trajectories and adds forward evaluations for caching and candidate scoring. Recorded caching and round-finalization time accounts for 8.44--15.17\% of the training-progress interval across the reported runs. Finalization includes scoring and saving; this fraction is not a measured slowdown relative to FedAvg. Appendix~\ref{app:efficiency} gives the timing breakdown and forward-pass accounting.

\subsection{Rollout Sources and Subsequent Learning}
\label{sec:mechanism_experiments}
The theory identifies the rollout source as a mediator between aggregation and future learning. Two receiving students start from the same round-9 checkpoint and train at the same local learning rate. Their data come from sources $\phi_1=\theta_t+\Delta_t$ and $\phi_3=\theta_t+3\Delta_t$, constructed from the same round-10 increment. The intervention changes the generator supplying training data while holding the receiving students' starting point fixed (Appendix~\ref{app:source_protocol}).

From a common start of 8.97\% Avg@8 and 46.20\% Pass@8, source $a=3$ yields a receiver scoring 10.17\% and 53.00\%; source $a=1$ yields 9.18\% and 44.60\% (Figure~\ref{fig:source_intervention}). Changing the source changes subsequent learning from the same receiving model, supporting the proposed rollout-mediated mechanism.

\begin{figure}[!t]
\centering
\includegraphics[width=.86\linewidth]{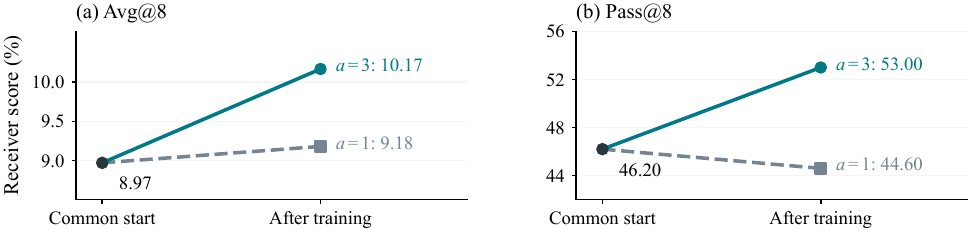}
\caption{\textbf{Rollout source changes subsequent learning.} MATH500 scores before and after short receiver training with source $a=1$ or $a=3$. Segments connect the common start to the two measured endpoints; they are not intermediate training trajectories. Both branches share the initial checkpoint and local learning rate.}
\label{fig:source_intervention}
\end{figure}

%% file: tables/main_avg.tex
\begin{table}[!t]
\centering\small
\caption{\textbf{Avg@8 across benchmarks, learning rates, and student scales.} Scores are percentages; Macro weights all six benchmarks equally. Best/second-best scores within each setting are bold/underlined. Base is repeated for reference; Olympiad denotes OlympiadBench.}
\label{tab:main}
\renewcommand{\arraystretch}{1.06}
\setlength{\tabcolsep}{4pt}
\begin{adjustbox}{max width=\linewidth}
\begin{tabular}{lrrrrrrr}
\toprule
\rowcolor{headbg}
\textbf{Method} & \textbf{MATH500} & \textbf{AMC23} & \textbf{Minerva} & \textbf{Olympiad} & \textbf{AIME24} & \textbf{AIME25} & \textbf{Macro}\\
\midrule
\multicolumn{8}{l}{\textit{0.6B student, 1$\times$LR}}\\
Base student & 9.07 & 4.06 & 2.85 & 3.12 & 0.42 & 0.42 & 3.32\\
Centralized & 31.70 & 16.88 & 7.95 & 10.31 & 0.42 & \underline{0.83} & 11.35\\
Local mean & 12.35 & 3.96 & 3.57 & 4.32 & 0.42 & 0.14 & 4.13\\
FedAvg & 11.75 & 3.44 & 3.95 & 4.30 & 0.00 & 0.00 & 3.91\\
Fixed $a=3$ & 31.80 & 14.37 & 8.04 & 9.40 & 0.42 & \underline{0.83} & 10.81\\
Fixed $a=10$ & \underline{44.12} & \underline{21.25} & \underline{12.87} & \textbf{15.56} & \textbf{2.08} & \textbf{1.25} & \underline{16.19}\\
\rowcolor{oursbg}\textbf{FedTOPS} & \textbf{45.98} & \textbf{24.06} & \textbf{13.83} & \underline{15.50} & \underline{0.83} & \underline{0.83} & \textbf{16.84}\\
\midrule
\multicolumn{8}{l}{\textit{0.6B student, 3$\times$LR}}\\
Base student & 9.07 & 4.06 & 2.85 & 3.12 & 0.42 & 0.42 & 3.32\\
Centralized & 45.02 & \underline{20.94} & \underline{13.83} & 15.34 & \underline{1.67} & \textbf{2.08} & \underline{16.48}\\
Local mean & 26.77 & 13.12 & 7.08 & 8.04 & 1.11 & 0.69 & 9.47\\
FedAvg & 37.28 & 16.56 & 9.56 & 11.18 & \underline{1.67} & 0.42 & 12.78\\
Fixed $a=3$ & \underline{45.20} & 20.00 & \textbf{14.38} & \underline{15.49} & \underline{1.67} & \underline{1.67} & 16.40\\
Fixed $a=10$ & 32.40 & 14.69 & 8.73 & 10.00 & 0.83 & 1.25 & 11.32\\
\rowcolor{oursbg}\textbf{FedTOPS} & \textbf{47.02} & \textbf{22.81} & \underline{13.83} & \textbf{17.41} & \textbf{2.08} & 0.83 & \textbf{17.33}\\
\midrule
\multicolumn{8}{l}{\textit{1.7B student, 1$\times$LR}}\\
Centralized & \underline{62.10} & \textbf{34.69} & \underline{23.30} & \underline{27.17} & \textbf{6.67} & \underline{3.33} & \underline{26.21}\\
Local mean & 29.41 & 15.62 & 10.09 & 11.99 & 2.92 & 1.39 & 11.90\\
FedAvg & 28.82 & 15.94 & 10.20 & 11.29 & 2.50 & 1.67 & 11.74\\
\rowcolor{oursbg}\textbf{FedTOPS} & \textbf{62.38} & \underline{33.75} & \textbf{23.71} & \textbf{27.58} & \underline{6.25} & \textbf{4.17} & \textbf{26.31}\\
\bottomrule\end{tabular}\end{adjustbox}\end{table}

%% file: tables/main_pass.tex
\begin{table}[!t]
\centering\small
\caption{\textbf{Pass@8 across benchmarks, learning rates, and student scales.} Scores are percentages. Settings, abbreviations, and highlighting follow Table~\ref{tab:main}.}
\label{tab:main_pass}
\renewcommand{\arraystretch}{1.06}
\setlength{\tabcolsep}{4pt}
\begin{adjustbox}{max width=\linewidth}
\begin{tabular}{lrrrrrrr}
\toprule
\rowcolor{headbg}
\textbf{Method} & \textbf{MATH500} & \textbf{AMC23} & \textbf{Minerva} & \textbf{Olympiad} & \textbf{AIME24} & \textbf{AIME25} & \textbf{Macro}\\
\midrule
\multicolumn{8}{l}{\textit{0.6B student, 1$\times$LR}}\\
Base student & 46.80 & 22.50 & 16.54 & 16.91 & 3.33 & \underline{3.33} & 18.24\\
Centralized & 68.40 & 50.00 & 26.47 & 33.68 & 3.33 & \textbf{6.67} & 31.43\\
Local mean & 51.93 & 22.50 & 18.14 & 22.55 & 2.22 & 1.11 & 19.74\\
FedAvg & 51.80 & 22.50 & 18.75 & 21.96 & 0.00 & 0.00 & 19.17\\
Fixed $a=3$ & 69.00 & 42.50 & 28.31 & 31.90 & 3.33 & \textbf{6.67} & 30.28\\
Fixed $a=10$ & \underline{71.40} & \underline{52.50} & \underline{30.15} & \underline{35.91} & \textbf{13.33} & \textbf{6.67} & \textbf{34.99}\\
\rowcolor{oursbg}\textbf{FedTOPS} & \textbf{72.40} & \textbf{57.50} & \textbf{31.62} & \textbf{36.65} & \underline{6.67} & \underline{3.33} & \underline{34.69}\\
\midrule
\multicolumn{8}{l}{\textit{0.6B student, 3$\times$LR}}\\
Base student & 46.80 & 22.50 & 16.54 & 16.91 & 3.33 & 3.33 & 18.24\\
Centralized & 71.00 & \underline{52.50} & 31.25 & \underline{36.65} & \underline{6.67} & \underline{10.00} & 34.68\\
Local mean & 67.00 & 45.00 & 24.88 & 30.42 & \textbf{7.78} & 4.44 & 29.92\\
FedAvg & 72.00 & 45.00 & 29.04 & 33.68 & \underline{6.67} & 3.33 & 31.62\\
Fixed $a=3$ & \textbf{73.00} & \textbf{55.00} & \underline{31.99} & 35.76 & \underline{6.67} & \textbf{13.33} & \textbf{35.96}\\
Fixed $a=10$ & 69.80 & 42.50 & 29.04 & 32.49 & 3.33 & 6.67 & 30.64\\
\rowcolor{oursbg}\textbf{FedTOPS} & \underline{72.20} & \textbf{55.00} & \textbf{32.72} & \textbf{39.76} & \underline{6.67} & 6.67 & \underline{35.50}\\
\midrule
\multicolumn{8}{l}{\textit{1.7B student, 1$\times$LR}}\\
Centralized & \textbf{84.20} & \underline{60.00} & \underline{40.44} & \underline{50.74} & \underline{13.33} & \underline{13.33} & \underline{43.67}\\
Local mean & 75.20 & 55.83 & 35.17 & 40.26 & \textbf{14.44} & 10.00 & 38.48\\
FedAvg & 76.20 & 50.00 & 34.19 & 39.47 & 10.00 & 10.00 & 36.64\\
\rowcolor{oursbg}\textbf{FedTOPS} & \underline{83.60} & \textbf{67.50} & \textbf{41.54} & \textbf{52.82} & \underline{13.33} & \textbf{23.33} & \textbf{47.02}\\
\bottomrule\end{tabular}\end{adjustbox}\end{table}

%% file: sections/related_work.tex
\section{Related Work}
\label{sec:related}
\paragraph{Federated low-rank adaptation.}
FFA-LoRA identifies a mismatch between jointly optimizing LoRA factors locally and averaging them separately, motivating a frozen-factor design \citep{sun2024ffa}. FLoRA uses stacking to aggregate heterogeneous adapters \citep{wang2024flora}, while FedEx-LoRA incorporates the aggregation residual into frozen weights \citep{singhal2025fedexlora}. These works address the representation of aggregated updates. FedOPD introduces a further coupling: the aggregated student generates the next training trajectories. We study how optimization discrepancies affect this evolving supervision and use teacher feedback to calibrate the existing update.

\paragraph{On-policy distillation and model-dependent data.}
Learning from states visited by the learner is central to interactive imitation learning \citep{ross2011dagger} and language-model OPD \citep{agarwal2024gkd,gu2024minillm}. Recent work studies teacher compatibility \citep{li2026rethinking}, entropy-aware supervision \citep{jin2026entropy}, data selection \citep{hou2026what}, and the distinction between state coverage and supervision absorption \citep{fu2026oneexample}. More broadly, performative prediction and performative FL formalize model--data feedback \citep{perdomo2020performative,jin2024performative}; federated SARSA studies evolving policies in heterogeneous environments \citep{zhang2024fedsarsa,mangold2025fedsarsa}. Our analysis identifies a specific consequence for FedOPD: aggregation lag can delay access to supervision, giving learning rate two coupled roles in finite-budget learning.

\paragraph{Federated optimization and extrapolation.}
FedAvg combines local updates \citep{mcmahan2017fedavg}, and SCAFFOLD addresses heterogeneous client drift \citep{karimireddy2020scaffold}. FedExP selects a server step size from client-update geometry \citep{jhunjhunwala2023fedexp}. In centralized OPD, EffOPD extrapolates checkpoint displacements using validation feedback \citep{cai2026foresee}. \method\ selects the magnitude of the current FedAvg increment using cached teacher distributions and clientwise predictive-change constraints. Its frozen-prefix score and bounded predictive change are related to trust-region policy optimization \citep{schulman2015trpo}. The motivation here is the effect of aggregation on both current progress and subsequent supervision.

\section{Limitations}
Our theory explains aggregation–rollout feedback without quantitatively predicting LLM dynamics. Experiments are limited to mathematical reasoning, three fully participating clients, one model family, and one training seed per configuration. FedTOPS scales LoRA factor updates and uses cached teacher agreement as a proxy for future learning benefits. Inexact LoRA aggregation and nonlinear weight-space extrapolation lie outside our analysis; integrating aggregation corrections, more OPD objectives and exploring broader settings remain future work.

%% file: sections/conclusion.tex
\section{Conclusion}
\label{sec:conclusion}
Federated on-policy distillation couples model aggregation with the generation of future training data. We find that this coupling can make the benefit of collaboration strongly sensitive to learning rate. Our analysis explains how an aggregation-induced optimization lag can delay access to teacher supervision, giving learning rate a second role beyond learning from the current data. This insight motivates \method, a simple teacher-guided adjustment of the FedAvg update magnitude that improves performance across the evaluated student models and learning rates. More broadly, effective aggregation in FedOPD must account for both the progress made by the current update and the learning opportunities created by the resulting student.

%% file: sections/statements.tex
\subsection*{AI Use Statement}
We used generative AI tools to polish an existing author-written draft, improving clarity, readability, and linguistic expression. We also used AI tools to assist with code development. The authors reviewed and verified all AI-assisted text and code and take full responsibility for the final content of the paper.

\subsection*{Reproducibility Statement}
Appendices~\ref{app:general}--\ref{app:intervention} provide assumptions and complete proofs. The source package includes theory verification, figure-generation code, machine-readable aggregate transcriptions, and table-generation code. Appendix~\ref{app:protocol} specifies the training and evaluation protocol, Appendix~\ref{app:additional} reports supplementary analyses, and Appendix~\ref{app:mainresults} gives complete benchmark results. The bundled experimental figures and tables can be rebuilt from the supplied data and reported summaries. We also upload the code as the supplementary materials for reproduction.

%% file: appendices/notation.tex
\section{Notation and Organization of the Analysis}
\label{app:notation}
We distinguish the receiving learner $\theta$ from the trajectory generator $\phi$ throughout the analysis. Table~\ref{tab:notation} summarizes the notation. Appendix~\ref{app:general} derives the update decomposition and propagation of learner--generator discrepancies. Appendix~\ref{app:construction} gives the solvable model, its exact recurrences, and the source intervention. Appendix~\ref{app:intervention} analyzes aggregate scaling and the relation between full-distribution and finite-cache trust constraints.

\begin{table}[htbp]
\centering
\caption{Main notation. $\Pmap$ denotes the matched-depth synchronous reference used in the analysis; Centralized denotes the pooled-data training baseline in the experiments.}
\label{tab:notation}
\begin{tabularx}{\linewidth}{lY}
\toprule

\rowcolor{headbg}\textbf{Symbol} & \textbf{Meaning}\\
\midrule
$K,w_i,D_i$ & Number of clients, aggregation weight, and fixed prompt distribution\\
$q,\pi_\theta,\pi_\phi$ & Fixed teacher, receiving learner, and trajectory generator\\
$L,d_i^\phi$ & Maximum horizon and uniformly position-sampled prefix distribution\\
$L_i(\theta;\phi),g_i(\theta;\phi)$ & Fixed-source loss and its learner gradient\\
$H_i,B_i$ & Learner response and generator response of the update vector field\\
$E,\eta,T$ & Local gradient steps per round, step size, and communication rounds\\
$\F,\Pmap$ & Local-then-average map and matched-depth synchronous map\\
$\delta_t,e_t$ & Cross-process state gap and one-round algorithm discrepancy\\
$u,v$ & Prerequisite and downstream parameters in the construction\\
$h_i,b,c,\lambda$ & Client curvature, shared prerequisite target, access threshold, and downstream precision\\
$\rho_i(u)$ & Probability of reaching a downstream teaching state\\
$r_\eta,s_\eta(u)$ & Exact prerequisite and downstream residual factors\\
$\kappa_\eta,\cC_T$ & Access sensitivity and effective supervision clock\\
$\Delta_t,a,\delta$ & FedAvg increment, server multiplier, and predictive KL budget\\
\bottomrule
\end{tabularx}
\end{table}

%% file: appendices/general_proofs.tex
\section{General mechanism: assumptions and complete derivations}
\label{app:general}
\subsection{The learner and source are different arguments}
For a prompt $x$ and a length-$L$ padded trajectory, write
\begin{equation}
P_i^\phi(x,y_{1:L})=D_i(x)\prod_{h=1}^{L}\pi_\phi(y_h\mid x,y_{<h}).
\label{eq:pathlaw}
\end{equation}
After termination, both policies use a common absorbing symbol, and the corresponding losses and score derivatives are zero. Drawing $h$ uniformly from $\{1,\ldots,L\}$ and retaining $(x,y_{<h},h)$ defines $d_i^\phi$. For continuous outputs, products and expectations refer to densities and integrals. The fixed-source objective is
\begin{equation}
L_i(\theta;\phi)=\frac1L\sum_{h=1}^{L}\E_{P_i^\phi}\ell(\theta;x,y_{<h},h).
\end{equation}
The conditional loss can be either exact forward or reverse KL; the following derivations only require the stated regularity. In particular,
\begin{equation}
\nabla_\theta[L_i(\theta;\theta)]
=g_i(\theta;\theta)+\left.\partial_\phi L_i(\theta;\phi)\right|_{\phi=\theta}.
\label{eq:diagonalrisk}
\end{equation}
Frozen-prefix training uses the first term. It does not generally optimize the diagonal risk by its full gradient.

\begin{assumption}[Local regularity and round protocol]
\label{ass:regularity}
The client set and positive weights are fixed. Each round uses all clients, a common starting learner, and $E$ exact gradient steps with one frozen source per client. There is a compact parameter neighborhood containing all points used in the local expansions and their intervening line segments. The functions $g_i(\theta;\phi)$ have bounded continuous derivatives through order two there, with locally Lipschitz second derivatives. The prefix laws have common support on the nonabsorbing states. A dominating integrable function justifies differentiating each expectation with respect to the learner and source. The bounds hold uniformly for the sufficiently small step sizes under consideration and for fixed $E$.
\end{assumption}

This is a sufficient local assumption, not a global boundedness assertion for neural-network KL losses. In the Gaussian construction all relevant expectations are analytic and these bounds hold on compact parameter sets. The fixed teacher supplies no parameter-dependent derivative.

\subsection{Proof of the learner--generator decomposition}
At the learner--source pair, define $H_i=\partial_\theta g_i$ and $B_i=\partial_\phi g_i$. The on-policy chain rule gives
\begin{equation}
Dg_i^\circ(\theta)=H_i(\theta;\theta)+B_i(\theta;\theta).
\label{eq:HB}
\end{equation}
For $s_h=(x,y_{<h},h)$ and $z_\theta(s)=\nabla_\theta\ell(\theta;s)$, the source response is
\begin{equation}
B_i(\theta;\theta)=\frac1L\sum_{h=1}^{L}\E\!\left[
 z_\theta(s_h)
 \left(\sum_{j<h}\nabla_\theta\log\pi_\theta(y_j\mid x,y_{<j})\right)^{\!\top}
\right].
\label{eq:score_response}
\end{equation}

Let $z_\theta(s)=\nabla_\theta\ell(\theta;s)$. At a fixed learner,
\begin{align}
\partial_\phi g_i(\theta;\phi)
&=\frac1L\sum_h\int z_\theta(s_h)
    [\nabla_\phi P_i^\phi(s_h)]^\top\,\mathrm ds_h\\
&=\frac1L\sum_h\E_{P_i^\phi}
 \left[z_\theta(s_h)\big(\nabla_\phi\log P_i^\phi(s_h)\big)^\top\right].
\end{align}
The prompt distribution has no parameter dependence, so
\begin{equation}
\nabla_\phi\log P_i^\phi(s_h)
=\sum_{j<h}\nabla_\phi\log\pi_\phi(y_j\mid x,y_{<j}).
\end{equation}
Evaluation at $\phi=\theta$ proves Eq.\eqref{eq:score_response}. Applying the chain rule to $g_i(\theta;\theta)$ gives $Dg_i^\circ=H_i+B_i$. The first prediction position has no earlier generated decisions, so its source-score term is zero. None of these expressions implies that $B_i$ is positive semidefinite, symmetric, or harmful.

\subsection{Proof of the fixed-source aggregation discrepancy}
Fix $(\theta,\phi)$ and suppress those arguments in $g_i$ and $H_i$. Let $\theta_{i,0}=\theta$ and
\begin{equation}
\theta_{i,e+1}=\theta_{i,e}-\eta g_i(\theta_{i,e};\phi).
\end{equation}
For fixed $E$, bounded gradients give $\theta_{i,e}-\theta=O(\eta)$. Induction and Taylor expansion yield
\begin{equation}
\theta_{i,e}=\theta-e\eta g_i+\binom e2\eta^2H_i g_i+O(\eta^3).
\label{eq:localexpansion}
\end{equation}
Indeed, substituting the induction hypothesis into
$g_i(\theta_{i,e};\phi)=g_i+H_i(\theta_{i,e}-\theta)+O(\eta^2)$
adds $e\eta^2H_i g_i$ to the second-order term, and $\binom e2+e=\binom{e+1}2$. Averaging Eq.\eqref{eq:localexpansion} at $e=E$ gives
\begin{equation}
\F(\theta;\phi)=\theta-E\eta g+\binom E2\eta^2\sum_iw_iH_i g_i+O(\eta^3).
\end{equation}
The identical expansion for $E$ steps on the pooled frozen-source gradient gives
\begin{equation}
\Pmap(\theta;\phi)=\theta-E\eta g+\binom E2\eta^2Hg+O(\eta^3).
\end{equation}
Finally,
\begin{equation}
\sum_iw_i(H_i-H)(g_i-g)=\sum_iw_iH_i g_i-Hg,
\end{equation}
which proves Eq.\eqref{eq:local_gap}. For $E=1$, both maps are exactly $\theta-\eta\sum_iw_i g_i(\theta;\phi)$, so equality is exact rather than merely second-order. If all client gradient functions agree, the maps also agree for every $E$. Identical optima alone do not imply identical gradient functions.

\paragraph{Why the source term does not appear here.}
Within this protocol, $\phi$ is fixed while each local learner changes. Consequently the within-round curvature is $H_i$, not $H_i+B_i$. Replacing it by $H_i+B_i$ would analyze a different algorithm that refreshes the source within the local loop. Client drift itself is not a new OPD-specific result \citep{karimireddy2020scaffold}; the next derivation concerns its additional propagation channel.

\subsection{Proof of the cross-round propagation formula}
Write the diagonal maps as $F^\circ(\theta)=\F(\theta;\theta)$ and $P^\circ(\theta)=\Pmap(\theta;\theta)$. The smooth local-step composition and Assumption~\ref{ass:regularity} imply the derivative expansion
\begin{equation}
DF^\circ(\theta)=I-E\eta(H+B)+O(\eta^2),
\label{eq:mapderivative}
\end{equation}
where the remainder is uniform on the neighborhood. This follows either by differentiating the smooth update composition, or by applying Eq.\eqref{eq:localexpansion} with its uniform differentiated remainder. It does not rely on differentiating an uncontrolled big-$O$ bound.

At $\theta_t^P$, add and subtract $F^\circ(\theta_t^P)$:
\begin{align}
\delta_{t+1}
&=F^\circ(\theta_t^P+\delta_t)-F^\circ(\theta_t^P)
  +F^\circ(\theta_t^P)-P^\circ(\theta_t^P)\\
&=DF^\circ(\theta_t^P)\delta_t+e_t+O(\norm{\delta_t}^2).
\end{align}
Substitution of Eq.\eqref{eq:mapderivative} proves Eq.\eqref{eq:propagation}.

If both processes instead use the same externally specified source $\phi_t$, differentiation of $\F(\theta;\phi_t)$ gives $I-E\eta H_t+O(\eta^2)$. The first-order generator contribution is absent because the source does not move with the receiving learner. Static optimization discrepancies, finite-sample noise, and effects of different optimizer states need not disappear. The proposition is a local decomposition, not a global contraction or divergence theorem.

\subsection{A general bound on capability-specific supervision}
\begin{lemma}[Access-limited movement]
\label{lem:access}
Let $v$ be a block of the learner parameters and $\cS$ a set of prefixes. Suppose, throughout the relevant local trajectories,
\begin{equation}
\norm{\nabla_v\ell(\theta;s)}
\le G\ind_{s\in\cS}+\epsilon\ind_{s\notin\cS},\qquad 0\le\epsilon\le G.
\end{equation}
If $\rho_i(\phi_t)=\Pr_{d_i^{\phi_t}}(s\in\cS)$ and the source is fixed within each round, population local gradient descent followed by averaging satisfies
\begin{equation}
\norm{v_T-v_0}\le\eta E\sum_{t=0}^{T-1}\sum_iw_i
\left[\epsilon+(G-\epsilon)\rho_i(\phi_t)\right].
\label{eq:accessbound}
\end{equation}
\end{lemma}
\begin{proof}
For any local learner iterate, Jensen's inequality for the norm gives
$\norm{\nabla_v L_i(\theta;\phi_t)}\le\epsilon+(G-\epsilon)\rho_i(\phi_t)$.
Summing $E$ steps bounds the local movement. The convexity of the norm bounds the aggregated movement by the weighted local bounds, and the triangle inequality over rounds proves the claim.
\end{proof}
At $\epsilon=0$, no capability-related movement can be arbitrarily large without accumulating visits to the teaching states. With nonnegative server multipliers, the same argument inserts $a_t$ in each round's bound. The lemma does not apply unchanged to arbitrary momentum states. Its prefix-gradient assumption is capability-specific: it does not state that all incorrect or long trajectories are uninformative.

\subsection{An elementary long-horizon extension}
Let a trajectory-level teaching event be exactly $\cS=\bigcap_{h=1}^{H}A_h$, where $A_h$ are specified successive events. Define $p_h^j=\Pr_j(A_h\mid A_1,\ldots,A_{h-1})$ for $j\in\{F,P\}$, with positive conditioning probabilities and $p_h^P>0$. If $p_h^F\le r p_h^P$ for every $h$, with $0<r<1$, then the chain rule gives
\begin{equation}
\frac{\Pr_F(\cS)}{\Pr_P(\cS)}
=\prod_{h=1}^{H}\frac{p_h^F}{p_h^P}\le r^H.
\label{eq:complete_event_ratio}
\end{equation}
If the numerator is zero, the bound is immediate without defining conditionals beyond a zero-probability prefix. The equality requires a \emph{complete} event characterization: if $\cS$ is only a subset of $\bigcap_h A_h$, an additional conditional success factor is needed, and the prerequisite inequalities alone do not imply Eq.\eqref{eq:complete_event_ratio}.

With $N$ independent trajectories from a fixed generator, the probability of no visit to this event is $(1-\Pr(\cS))^N$. This illustrates how horizon and sampling budget can regulate rare-supervision access. It does not imply domination of all state occupancies or that long reasoning is inherently harmful.

\subsection{Relation to practical OPD implementations}
An exact conditional reverse KL fits Eq.\eqref{eq:twoarg}. A clipped importance-weighted surrogate can additionally depend explicitly on the behavior policy through its likelihood ratios, saved advantages, and clipping sets. Its source derivative must include those dependencies, and may require piecewise or generalized derivatives at clipping boundaries. Persistent Adam states require augmenting the state space beyond $\theta$. Likewise, if the rollout sampler uses temperature or top-$p$ truncation, the source-score expression must use that actual sampling law; the raw model softmax is not automatically the generator distribution. The controlled population experiments isolate the mechanism before these implementation effects are introduced.

%% file: appendices/construction_proofs.tex
\section{A solvable on-policy construction and its consequences}
\label{app:construction}
\subsection{Deriving the objective from a shared conditional teacher}
A client prompt is the deterministic feature $x_i>0$. The teacher is a single conditional function of $(x,Y)$, not a different target chosen independently for each client. At the first prediction it uses $\mathcal N(xb,1)$, while the student uses $\mathcal N(xu,1)$. At the second prediction, if $Y\ge xc$, the teacher uses $\mathcal N(1,\lambda^{-1})$ and the student uses $\mathcal N(v,\lambda^{-1})$. Otherwise both use a common distribution independent of $(u,v)$, for example $\mathcal N(0,\lambda^{-1})$.

For equal-variance Gaussians,
\begin{equation}
\KL\big(\mathcal N(m_1,\sigma^2)\|\mathcal N(m_2,\sigma^2)\big)
=\frac{(m_1-m_2)^2}{2\sigma^2}.
\end{equation}
The first-position conditional loss is $x_i^2(u-b)^2/2$. Under a source $\phi$, the probability of the second teaching state is
\begin{equation}
\Pr(Y\ge x_i c\mid Y\sim\mathcal N(x_i\phi_u,1))
=\Phi(x_i(\phi_u-c)).
\end{equation}
The second-position conditional loss, whenever reached, is $\lambda(v-1)^2/2$. Summing both positions proves Eq.\eqref{eq:construction_loss}. This convention differs from the position average in Appendix~\ref{app:general} by the constant factor two; the step size absorbs that factor. Equality of forward and reverse KL here is an equality of the \emph{conditional} losses. It does not assert equality of the two joint trajectory KLs evaluated under different source laws.

For every finite $\phi_u$, $\rho_i(\phi_u)>0$. Thus the frozen-source Hessian is diagonal with positive entries $h_i$ and $\lambda\rho_i(\phi_u)$, and the unique optimum is $(b,1)$ for every client. There is no teacher incompatibility and no conflict between client optima. The threshold is fixed in the generated variable; the expected objective is smooth in its parameter arguments.

\subsection{Exact iterates for all four processes}
Let $h_i=x_i^2$, $u_0<c<b$, $v_0\in[0,1)$, and
\begin{equation}
0<\eta\max\{h_1,\ldots,h_K,\lambda\}<1.
\label{eq:stablecondition}
\end{equation}
With the round-$t$ source fixed, the local gradient is
\begin{equation}
g_i((u,v);\phi_t)=\big(h_i(u-b),\lambda\rho_i(\phi_{t,u})(v-1)\big)^\top.
\end{equation}
Starting at $(u_t,v_t)$, the $E$-step residuals on client $i$ are exactly
\begin{equation}
 b-u_{i,E}=(1-\eta h_i)^E(b-u_t),\qquad
 1-v_{i,E}=(1-\eta\lambda\rho_i(\phi_{t,u}))^E(1-v_t).
\label{eq:exactlocal}
\end{equation}
For FedAvg, $\phi_{t,u}=u_t^F$. Averaging Eq.\eqref{eq:exactlocal} gives Eq.\eqref{eq:exact_recursion} and hence
\begin{equation}
 u_t^F=b-(b-u_0)r_\eta^t,\qquad
 1-v_T^F=(1-v_0)\prod_{t=0}^{T-1}s_\eta(u_t^F).
\label{eq:closedFed}
\end{equation}
For the synchronous reference, each of the $E$ gradient steps uses the mixed frozen-source loss. Its exact factors are $r_\eta^P$ and $s_\eta^P$ in the main text, so
\begin{equation}
 u_t^P=b-(b-u_0)(r_\eta^P)^t,\qquad
 1-v_T^P=(1-v_0)\prod_{t=0}^{T-1}s_\eta^P(u_t^P).
\label{eq:closedPooled}
\end{equation}
Independent Local training has
\begin{align}
 u_{i,t}^L&=b-(b-u_0)(1-\eta h_i)^{Et},\label{eq:local_u}\\
 1-v_{i,T}^L&=(1-v_0)\prod_{t=0}^{T-1}
 (1-\eta\lambda\rho_i(u_{i,t}^L))^E.\label{eq:local_v}
\end{align}
Finally, the source-swapped federated learner retains the $E$-step local updates and averaging, but uses the predetermined $u_t^P$ source schedule. The generator's second coordinate is irrelevant to the distribution of prefixes at the second prediction. Therefore
\begin{equation}
1-v_T^{\mathrm{swap}}=(1-v_0)\prod_{t=0}^{T-1}s_\eta(u_t^P).
\label{eq:closedSwap}
\end{equation}
The swapped experiment changes supervision generation, not the receiving learner's number of updates or its update rule.

\subsection{Proof of Theorem~\ref{thm:bottleneck}}
\paragraph{Prerequisite lag.}
For $E\ge2$, the second derivative of $f(z)=(1-\eta z)^E$ is
$E(E-1)\eta^2(1-\eta z)^{E-2}\ge0$ on the stable interval. Jensen's inequality yields
\begin{equation}
r_\eta=\sum_iw_i f(h_i)\ge f\left(\sum_iw_i h_i\right)=r_\eta^P.
\end{equation}
Both factors lie in $(0,1)$, so Eq.\eqref{eq:closedFed}--Eq.\eqref{eq:closedPooled} imply $u_t^F\le u_t^P$. For $E=2$ the Jensen gap is exact:
\begin{equation}
r_\eta-r_\eta^P
=\eta^2\left[\sum_iw_i h_i^2-\left(\sum_iw_i h_i\right)^2\right]
=\eta^2\Var_w(h_i).
\end{equation}
From a common starting point $u$, the one-round prerequisite difference is therefore $\eta^2\Var_w(h_i)(b-u)$.

\paragraph{The two downstream costs.}
The Gaussian CDF has derivative
\begin{equation}
\rho_i'(u)=\frac{x_i}{\sqrt{2\pi}}
\exp\!\left[-\frac{x_i^2(u-c)^2}{2}\right]>0.
\label{eq:rho_derivative}
\end{equation}
Consequently $s_\eta(u)$ decreases with $u$. Jensen's inequality applied to $\rho_i(u)$ also gives
\begin{equation}
s_\eta(u_t^F)\ \ge\ s_\eta(u_t^P)\ \ge\ s_\eta^P(u_t^P).
\label{eq:two_costs}
\end{equation}
All factors and the initial residual $1-v_0$ are positive. Multiplying over $t$ proves
$v_T^F\le v_T^{\mathrm{swap}}\le v_T^P$.
The first comparison changes only the source; the second compares local-then-average and synchronous learning from the same source schedule.

\paragraph{Exact source attribution.}
Direct differentiation gives
\begin{equation}
\kappa_\eta(u)
=\frac{E\eta\lambda\sum_iw_i\rho_i'(u)
  (1-\eta\lambda\rho_i(u))^{E-1}}{s_\eta(u)}>0.
\label{eq:kappa_expanded}
\end{equation}
Taking the logarithm of the ratio of Eq.\eqref{eq:closedFed} and~Eq.\eqref{eq:closedSwap} yields
\begin{align}
\log\frac{1-v_T^F}{1-v_T^{\mathrm{swap}}}
&=\sum_t\big[\log s_\eta(u_t^F)-\log s_\eta(u_t^P)\big]\\
&=\sum_t\int_{u_t^F}^{u_t^P}\kappa_\eta(z)\,\mathrm dz.
\end{align}
If each $\rho_i$ is replaced by a model-independent constant, $s_\eta$ no longer depends on $u$, $\kappa_\eta=0$, and the source-swapped and original federated $v$ iterates coincide exactly. Constants can still differ across clients, so ordinary Jensen discrepancies between local and synchronous learning can remain.

\paragraph{Monotonic convergence, not divergence.}
Since $0<r_\eta<1$, $u_t^F$ increases to $b$. Because $u_t^F\ge u_0$,
\begin{equation}
0<s_\eta(u_t^F)\le s_\eta(u_0)<1.
\end{equation}
The uniform upper bound follows from $\rho_i(u_0)>0$. Hence
$0\le1-v_T^F\le(1-v_0)s_\eta(u_0)^T\to0$.
The same argument applies to the pooled and Local recurrences and also to the swapped process. No step overshoots its shared target under Eq.\eqref{eq:stablecondition}. This proves all statements of Theorem~\ref{thm:bottleneck}.

\subsection{The coupling is present without an unstable eigenvalue}
The exact federated map has Jacobian
\begin{equation}
D\F(u,v)=
\begin{pmatrix}
r_\eta & 0\\
-(1-v)s_\eta'(u) & s_\eta(u)
\end{pmatrix}.
\label{eq:triangularJac}
\end{equation}
Both eigenvalues lie in $(0,1)$, and the lower-left entry is nonnegative when $v\le1$. An upstream lag can therefore affect downstream learning without an eigenvalue exceeding one. The matrix need not be normal or a Euclidean contraction; monotonic convergence above was established directly, rather than inferred from a local eigenvalue calculation. At the gradient level, the source response is particularly transparent:
\begin{equation}
H_i=\begin{pmatrix}h_i&0\\0&\lambda\rho_i(u)\end{pmatrix},\qquad
B_i=\begin{pmatrix}0&0\\\lambda\rho_i'(u)(v-1)&0\end{pmatrix}.
\end{equation}
The generator term transmits prerequisite differences to the downstream coordinate.

\subsection{A fixed-source contrast against independent Local training}
\begin{corollary}[No negative gain in the constant-access counterfactual]
\label{cor:static_local}
Replace the access laws in Eq.\eqref{eq:construction_loss} by client-specific constants $p_i\in[0,1]$, held fixed across rounds, learning rates, and training processes. Under Eq.\eqref{eq:stablecondition}, use the same $E,T,\eta,w_i$, and initial downstream parameter $v_0$ for FedAvg and independent Local training. On the common evaluation score $S(v)=1-(1-v)^2$,
\begin{equation}
S(v_T^F)-\sum_iw_iS(v_{i,T}^L)
=(1-v_0)^2\left[
\sum_iw_i\zeta_i^{2T}-\left(\sum_iw_i\zeta_i\right)^{2T}
\right]\ge0,
\label{eq:static_gain}
\end{equation}
where $\zeta_i=(1-\eta\lambda p_i)^E$.
\end{corollary}
\begin{proof}
The constant-access residuals are
$1-v_T^F=(1-v_0)(\sum_iw_i\zeta_i)^T$ and
$1-v_{i,T}^L=(1-v_0)\zeta_i^T$.
Substitute them into $S$. Since $T\ge1$ and $\zeta_i\in[0,1]$, convexity of $z\mapsto z^{2T}$ proves the inequality by Jensen's inequality.
\end{proof}
Thus the negative low-step federated gain in this construction cannot be reproduced by merely retaining different constant downstream curvatures. Its model-dependent access law matters. This contrast is specific to the scalar common-state score, shared initialization, and time-constant access rates; it is not a general superiority theorem for FedAvg on static data. An arbitrary time-varying source schedule need not satisfy the same comparison.

\subsection{Bounds and sensitivity of the supervision clock}
\label{app:clock}
Let $\gamma=\eta\lambda\in(0,1)$. From Eq.\eqref{eq:closedFed},
$1-v_T^F=(1-v_0)\exp(-\cC_T)$ with $\cC_T=-\sum_t\log s_\eta(u_t^F)$.
For any $\rho\in[0,1]$,
\begin{equation}
1-(1-\gamma\rho)^E
=\int_0^{\gamma\rho} E(1-z)^{E-1}\,\mathrm dz
\ge E\gamma\rho(1-\gamma)^{E-1}.
\end{equation}
Using $-\log s\ge1-s$ gives the lower bound below. For the upper bound, convexity of $-\log$ gives
\begin{align}
-\log s_\eta(u)
&\le-E\sum_iw_i\log(1-\gamma\rho_i(u))\\
&\le E\sum_iw_i\frac{\gamma\rho_i(u)}{1-\gamma\rho_i(u)}
\le\frac{E\gamma}{1-\gamma}\bar\rho(u).
\end{align}
Consequently,
\begin{equation}
E\gamma(1-\gamma)^{E-1}\sum_t\bar\rho(u_t^F)
\le\cC_T\le\frac{E\gamma}{1-\gamma}\sum_t\bar\rho(u_t^F).
\label{eq:clock_bounds}
\end{equation}
The history of accessible supervision therefore regulates the effective amount of downstream learning.

For the derivative, first observe
\begin{equation}
r_\eta'=-E\sum_iw_i h_i(1-\eta h_i)^{E-1}<0,\qquad
\frac{\mathrm du_t^F}{\mathrm d\eta}=-(b-u_0)t r_\eta^{t-1}r_\eta'\ge0,
\end{equation}
where the derivative for $t=0$ is zero. At fixed $u$,
\begin{equation}
-\partial_\eta\log s_\eta(u)=
\frac{E\lambda\sum_iw_i\rho_i(u)(1-\eta\lambda\rho_i(u))^{E-1}}{s_\eta(u)}\ge0.
\label{eq:direct_clock}
\end{equation}
The chain rule then proves Eq.\eqref{eq:clock_derivative}; its second term is nonnegative by Eq.\eqref{eq:kappa_expanded}. If every learning rate is assigned the same external source schedule, independent of $\eta$, that second term is absent. This comparison requires fixing the source across learning rates, not merely freezing each rate's own source within a round.

The first source index at which every client has access probability at least $1/2$ is
\begin{equation}
T_{\mathrm{hit}}=\left\lceil
\frac{\log((b-c)/(b-u_0))}{\log r_\eta}
\right\rceil
\label{eq:hitting}
\end{equation}
The first source index $t$ with $u_t^F\ge c$ satisfies
$r_\eta^t\le(b-c)/(b-u_0)$, which proves Eq.\eqref{eq:hitting}. Since $\rho_i(c)=1/2$ for every client, this is a half-access threshold, not the onset of nonzero supervision: $\rho_i(u)>0$ for every finite $u$. A run of $T$ rounds uses sources $u_0,\ldots,u_{T-1}$, so the number of rounds at or above this threshold is $\max\{T-T_{\mathrm{hit}},0\}$. A crossing at $t=T$ cannot influence that run's preceding updates. This is a finite-budget access boundary, not a singular phase transition. Increasing $\eta$ outside the non-overshooting interval has not been analyzed by these inequalities. Nor does the theory predict that a vanishing step size increases one-round aggregation error.

\subsection{Numerical Illustration}
\label{app:numerics}
All illustrations use the fixed parameters
\begin{equation}
\begin{gathered}
K=2,\quad w_1=w_2=\tfrac12,\quad(h_1,h_2)=(1,9),\quad E=4,\quad T=5,\\
b=10,\quad c=7.25,\quad\lambda=2,\quad u_0=v_0=0.
\end{gathered}
\end{equation}
The common evaluation state is a downstream teaching prefix, and its score is
$S(v)=1-(1-v)^2$. Thus the evaluation distribution is fixed independently of the training generator. Local mean is $\sum_iw_iS(v_{i,T}^L)$, not $S(\sum_iw_iv_{i,T}^L)$. These scores are not percentages or fitted language-model accuracy estimates.

\begin{table}[htbp]
\centering
\caption{Exact-recursion scores for the same construction. All update factors satisfy the stated stability condition.}
\label{tab:construction}
\begin{adjustbox}{max width=\linewidth}
\begin{tabular}{lrrrr}
\toprule

\rowcolor{headbg}\textbf{$\eta$} & \textbf{FedAvg} & \textbf{Local mean} & \textbf{Pooled} & \textbf{Source-swapped FedAvg}\\
\midrule
0.02 & 0.281532 & 0.311221 & 0.360625 & 0.360494\\
0.06 & 0.950782 & 0.575594 & 0.977725 & 0.977670\\
\bottomrule
\end{tabular}
\end{adjustbox}
\end{table}
The federated gain changes from $-0.029689$ to $+0.375187$. At the smaller step, the faster client can independently reach downstream teaching states before the repeatedly averaged student does. At the larger step, the federated student reaches those states sufficiently early to share the acquired downstream capability. The comparison is finite-budget; all processes eventually reach the shared optimum.

\begin{figure}[t]
\centering
\includegraphics[width=.78\linewidth]{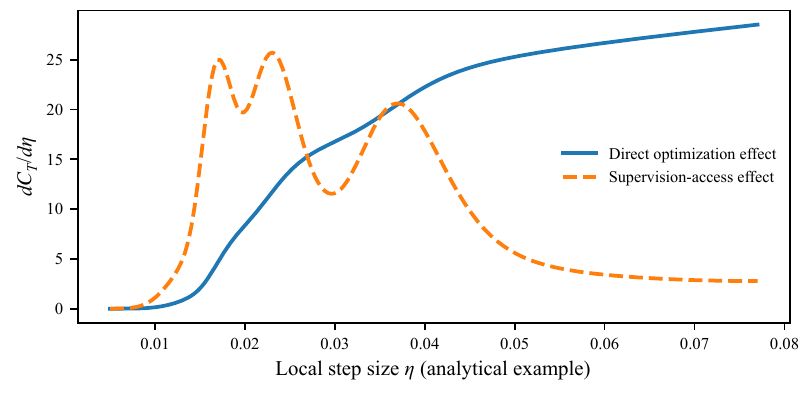}
\caption{\textbf{Two contributions to step-size sensitivity.} Both curves are exact derivatives of the construction's supervision clock. The mediated contribution changes the source along with the learner; it is absent under a source schedule held fixed across learning rates.}
\label{fig:clock_sensitivity}
\end{figure}
At $\eta=0.02$, the direct and mediated derivatives are approximately $8.368607$ and $19.770051$; at $\eta=0.06$, they are $26.682896$ and $3.410686$. The mediated contribution dominates at the smaller step size, whereas the direct contribution is larger at the higher step size, illustrating how their relative importance changes with access to downstream teaching states.

%% file: appendices/intervention_proofs.tex
\section{Magnitude calibration: constructive and local guarantees}
\label{app:intervention}
\subsection{Pure server scaling in the construction}
\begin{proposition}[A safe amplitude intervention]
\label{prop:scaling}
Under the construction's stability conditions, define
\begin{equation}
a_{\max}=\min\left\{\frac1{1-r_\eta},\frac1{1-(1-\eta\lambda)^E}\right\}.
\end{equation}
For a constant $a\in[1,a_{\max}]$, scaling the FedAvg increment by $a$ yields non-overshooting updates and, from the same initial state,
\begin{equation}
u_t^{(a)}\ge u_t^{(1)},\qquad v_t^{(a)}\ge v_t^{(1)}\quad\text{for all }t.
\end{equation}
Moreover, the population teacher loss evaluated on the current frozen source is nonincreasing in $a$ throughout this safe interval.
\end{proposition}
\begin{proof}
Let $\Delta u=(1-r_\eta)(b-u)$ and $\Delta v=(1-s_\eta(u))(1-v)$. The scaled residuals are
\begin{equation}
b-u^+=[1-a(1-r_\eta)](b-u),\qquad
1-v^+=[1-a(1-s_\eta(u))](1-v).
\label{eq:scaledresidual}
\end{equation}
Because $\rho_i(u)\le1$, we have $s_\eta(u)\ge(1-\eta\lambda)^E$. The bound on $a$ therefore places both bracketed factors in $[0,1]$. For the prerequisite, the scaled factor is no larger than $r_\eta$, so $u_t^{(a)}\ge u_t^{(1)}$ follows directly by induction. For the downstream coordinate, write $q_a(u)=1-a(1-s_\eta(u))$. If $u^{(a)}\ge u^{(1)}$, monotonicity of $s_\eta$ gives
\begin{equation}
0\le q_a(u^{(a)})\le s_\eta(u^{(a)})\le s_\eta(u^{(1)}).
\end{equation}
Multiplication by the ordered nonnegative residuals in Eq.\eqref{eq:scaledresidual} completes the induction.

On the frozen round-start source, the mixed teacher loss at the candidate is exactly
\begin{align}
J_t(a)={}&\tfrac{\bar h}{2}(b-u_t)^2[1-a(1-r_\eta)]^2\nonumber\\
&+\tfrac{\lambda\bar\rho(u_t)}2(1-v_t)^2[1-a(1-s_\eta(u_t))]^2.
\label{eq:construction_score}
\end{align}
Each squared term is nonincreasing while its bracket is nonnegative. Thus $J_t'(a)\le0$ on the safe interval.
\end{proof}
For the numerical instance at $\eta=0.02$, the fixed-state scores at $a=1,2,3$ are $0.281532$, $0.884641$, and $0.991698$. No new direction is introduced. Proposition~\ref{prop:scaling} does not guarantee the behavior of a finite-cache selector, which can choose $a<1$ and does not know this construction's safe interval. It demonstrates that correcting a progress deficit can be sufficient without correcting the direction.

In this construction, a full predictive-change score on a client's frozen prefixes has a particularly simple form:
\begin{equation}
K_{i,t}(a)=\tfrac{a^2}{2}\left[h_i(\Delta u_t)^2+
\lambda\rho_i(u_t)(\Delta v_t)^2\right].
\end{equation}
This uses the sum of both prediction-position KLs; dividing by two gives the position average. It makes explicit why a larger beneficial step can nevertheless be bounded by a behavior-change budget.

\subsection{Local scale compensation in a general parameterization}
At a fixed learner and optimizer state, suppose
$\Delta(\eta)=\eta d+\eta^2\xi_\eta$ with uniformly bounded $\xi_\eta$. For $\beta=a\eta$, the actual candidate displacement is
\begin{equation}
a\Delta(\eta)=\beta d+\beta\eta\xi_\eta.
\label{eq:scaled_direction_remainder}
\end{equation}
Thus extrapolation amplifies the original $O(\eta^2)$ direction error to $O(|\beta|\eta)$, rather than leaving it $O(\eta^2)$. For a smooth fixed-source teacher score along the limiting ray,
\begin{equation}
J(\theta+\beta d)=J(\theta)-b_0\beta+\tfrac12c_0\beta^2+o(\beta^2),
\end{equation}
where $b_0=-\nabla J(\theta)^\top d$ and $c_0=d^\top\nabla^2J(\theta)d$. In a favorable local model, $b_0,c_0>0$. Expanding the old-to-candidate KL yields
\begin{equation}
K_i(\theta+\beta d)=\tfrac12\beta^2 f_i+o(\beta^2),\qquad
f_i=d^\top F_i d\ge0,
\end{equation}
where $F_i$ is the conditional Fisher information averaged over the fixed source. These KL expansions require differentiable, normalized conditionals with common support and justified exchange of differentiation and expectation. For the actual candidate in Eq.\eqref{eq:scaled_direction_remainder}, the score expansion acquires an additional $O(|\beta|\eta)$ term and the KL expansion an $O(\beta^2\eta)$ term, locally for small $\eta$ and $|\beta|$; the latter uses the zero KL gradient at the starting policy. Consequently, the quadratic model discards both displacement and direction remainders. If $f_{\max}=\max_i f_i>0$, its continuous constrained problem over $\beta\ge0$ has the solution
\begin{equation}
\beta^\star=\min\left\{\frac{b_0}{c_0},\sqrt{\frac{2\delta}{f_{\max}}}\right\}.
\label{eq:local_beta}
\end{equation}
Hence the idealized compensating multiplier is $a^\star\simeq\beta^\star/\eta$.
This is a local quadratic design calculation, not an exact optimizer for the nonlinear problem. In particular, if $a$ grows like $1/\eta$, the candidate displacement must still remain in the neighborhood where the expansion is accurate; making $\eta$ small does not alone ensure that. A finite candidate grid, changed directions, clipping, and persistent optimizer states all limit compensation. If $b_0\le0$, a positive multiple of the proposed direction need not improve the teacher objective at all.

\subsection{What an ideal full-distribution trust constraint controls}
\label{app:trust}
The use of a frozen-state surrogate with a bounded policy change is related to trust-region policy optimization \citep{schulman2015trpo}. Here we provide the precise statement needed for the distillation setting, separately from its empirical implementation. Throughout, logarithms are natural and $\TV(P,Q)=\sup_A|P(A)-Q(A)|$, equivalently one half of the $L^1$ distance between densities.

\begin{proposition}[An ideal source-transfer bound]
\label{prop:trust}
Let $\theta$ be the old policy and $\theta'$ a candidate, with the same prompt distribution and padded horizon $L$. Define the full-vocabulary population constraint
\begin{equation}
K_i(\theta')=\E_{s\sim d_i^\theta}
\KL\big(\pi_\theta(\cdot\mid s)\|\pi_{\theta'}(\cdot\mid s)\big).
\end{equation}
Assume the same policies are used for sampling and for this KL, the trajectory laws have the necessary absolute continuity, and $0\le\ell(\theta';s)\le M$ on the relevant states. Then
\begin{equation}
\left|L_i(\theta';\theta')-L_i(\theta';\theta)\right|
\le M\sqrt{\frac{LK_i(\theta')}{2}}.
\label{eq:idealtrustbound}
\end{equation}
If $K_i(\theta')\le\delta$ for every client, then
\begin{equation}
L(\theta';\theta')\le L(\theta';\theta)+M\sqrt{\frac{L\delta}{2}}.
\end{equation}
\end{proposition}
\begin{proof}
The KL chain rule, with unchanged prompts, gives
\begin{equation}
\KL(P_i^\theta\|P_i^{\theta'})
=\sum_{h=1}^L\E_{P_i^\theta}
\KL(\pi_\theta(\cdot\mid s_h)\|\pi_{\theta'}(\cdot\mid s_h))
=LK_i(\theta').
\end{equation}
Sampling a uniform position and taking its prefix is a common Markov transformation of the trajectory. Data processing and Pinsker's inequality therefore imply
\begin{equation}
\TV(d_i^\theta,d_i^{\theta'})
\le\TV(P_i^\theta,P_i^{\theta'})
\le\sqrt{LK_i(\theta')/2}.
\end{equation}
The expectation of any function taking values in $[0,M]$ changes by at most $M$ times the total variation distance. Apply this to $\ell(\theta';\cdot)$, then average over clients.
\end{proof}
A corresponding change-from-start inequality is
\begin{equation}
L(\theta';\theta')-L(\theta;\theta)
\le \underbrace{L(\theta';\theta)-L(\theta;\theta)}_{\text{fixed-source score change}}
+M\sqrt{L\delta/2}.
\end{equation}
The bounded-loss assumption is additional and nontrivial for KL; a suitable probability floor or restricted domain can imply it, but the deployed neural model does not automatically satisfy it. For long sequences the bound can also be vacuous. The result bounds a distillation-loss distribution shift, not task accuracy or access to useful reasoning states.

\subsection{Finite-Cache and Vocabulary Approximations}
Fix a prefix and a set $S$ of retained token IDs, and combine the complement $T$ into one tail bucket. Let $p_T=\sum_{v\in T}p_v$, $q_T=\sum_{v\in T}q_v$, and let $\widetilde p,\widetilde q$ be the partitioned distributions. Expanding the tail sum gives the exact identity
\begin{equation}
\KL(p\|q)=\KL(\widetilde p\|\widetilde q)
+p_T\KL\big(p(\cdot\mid T)\|q(\cdot\mid T)\big),
\label{eq:tailidentity}
\end{equation}
when $p_T,q_T>0$, with extended-real KL if some conditional support is missing. If $p_T=0$, the tail term contributes zero. If $p_T>0=q_T$, both full and coarsened KL are infinite, and no undefined tail conditional is needed to assert the lower bound. Thus the coarsened KL is a lower bound. Its being below $\delta$ cannot certify that the full KL is below $\delta$, and changes within the tail are invisible to the coarsened score. Comparing candidates requires the same fixed partition for every candidate at that prefix.

The implemented cache also differs from a population prefix average: it uses few trajectories, regularly selected positions, and training-dependent candidates. Positions within a trajectory are correlated. Because the same round's data help create the candidates, a naive independent-validation finite-candidate bound is inapplicable.

Finally, the current rule does not always compare a feasible update with $a=0$. Even an accurately measured score better than the ordinary FedAvg candidate need not be better than the round start. Proposition~\ref{prop:trust} thus bounds source-distribution shift under its stated population assumptions; it does not establish monotonic improvement for the finite-cache selection rule.

%% file: appendices/algorithm_details.tex
\section{Algorithm and Implementation Details}
\label{app:algorithm}
Algorithm~\ref{alg:fedtops} gives one round of \method. Local OPD determines the update direction, and the teacher-guided selector chooses its magnitude using a fixed cache of that round's training trajectories.
\begin{algorithm}[htbp]
\small
\caption{\method: teacher-guided magnitude selection (one round)}
\label{alg:fedtops}
\begin{algorithmic}[1]
\Require Starting model $\theta_t$, client weights $w_i$, teacher $q$, budget $\delta$, candidates $\cA$
\For{each client $i$ in parallel}
 \State Generate local training trajectories from $\theta_t$; cache sampled prefixes and distributions
 \State Perform local OPD with learning rate $\eta$ to obtain $\theta_{i,t}^{+}$
\EndFor
\State $\Delta_t\gets\sum_i w_i\theta_{i,t}^{+}-\theta_t$
\For{$a\in\cA$}
 \State Form $\theta_t(a)\gets\theta_t+a\Delta_t$
 \State Evaluate $\widehat J_t(a)$ and each $\widehat K_{i,t}(a)$ on the fixed caches using Eq.\eqref{eq:selector}
\EndFor
\State $\cA_f\gets\{a\in\cA:\max_i\widehat K_{i,t}(a)\le\delta\}$
\If{$\cA_f=\varnothing$}
 \State $a_t\gets0$
\ElsIf{$1\in\cA_f$}
 \State $a_t\gets1$
 \For{$a\in\cA_f$ in fixed order}
  \If{$\widehat J_t(a)<\widehat J_t(a_t)$}
   \State $a_t\gets a$
  \EndIf
 \EndFor
\Else
 \State $a_t\gets\argmin_{a\in\cA_f}\widehat J_t(a)$
\EndIf
\State \Return $\theta_{t+1}=\theta_t+a_t\Delta_t$
\end{algorithmic}
\end{algorithm}
\subsection{Trajectory Cache and Candidate Scores}
Every client begins a round from the same global student. Before local optimization, it randomly selects four training responses and at most 16 evenly spaced prediction positions per response. All candidates use the same selected prefixes. Reusing the training responses avoids additional trajectory generation; candidate evaluation requires forward passes but no optimizer steps.

\paragraph{A fixed vocabulary partition.}
At each selected prefix, let $S$ be the union of the teacher's and starting student's top-16 token IDs. The cache stores their probabilities on $S$ and a tail bucket containing the remaining probability mass. For any distribution $p$, this tail is $1-\sum_{v\in S}p_v$, where $p_v$ is the full-softmax probability; the retained tokens are not renormalized. The same partition is used for every candidate.

Teacher agreement is measured by the teacher-to-candidate forward KL, averaged within each client and then combined using the aggregation weights. Predictive change is measured by the starting-student-to-candidate forward KL, with a separate constraint on each client's cache mean. Both quantities refer to these compressed distributions. 

\subsection{Selection Rule}
The candidate set is $\{0.5,1,2,3,5,10\}$. If $a=1$ is feasible, it is the incumbent and is replaced only by a feasible candidate with a strictly lower teacher score; equal scores therefore preserve FedAvg. If $a=1$ is infeasible, the server chooses the lowest-score feasible candidate. If no candidate is feasible, the server retains the round's starting model. Score-comparison and feasibility tolerances are $10^{-9}$ and $10^{-8}$, respectively; Algorithm~\ref{alg:fedtops} omits these tolerances for clarity. 

\subsection{Coordinate-Level Meaning for LoRA}
Aggregation and scaling operate on the uploaded LoRA factors, which share an initialization. We use rank 8, $\alpha=16$, zero dropout, and all-linear target modules. If a layer has factors $B_t,A_t$ and increments $\Delta B,\Delta A$, then
\begin{equation}
(B_t+a\Delta B)(A_t+a\Delta A)-B_tA_t
=a(\Delta B A_t+B_t\Delta A)+a^2\Delta B\Delta A.
\end{equation}
A scaled factor increment therefore need not equal a scaled merged dense update. The local-response analysis can be expressed in trainable coordinates, while the scalar construction abstracts away this factor-product geometry. \method\ uses no additional correction direction. We use PEFT's default LoRA initialization (Kaiming-uniform $A$ and zero-initialized $B$), shared across clients, and target all attention and MLP linear projections (\texttt{q\_proj}, \texttt{k\_proj}, \texttt{v\_proj}, \texttt{o\_proj}, \texttt{gate\_proj}, \texttt{up\_proj}, and \texttt{down\_proj}), excluding the language-model head.
Aggregation and server scaling operate directly on the LoRA factors without canonicalization or re-factorization, and evaluation loads the resulting adapter through vLLM without merging it into the frozen base weights.

\subsection{Execution and Data Locality}
Our experiments simulate logical clients on one GPU. In a distributed implementation, each client could evaluate candidates on its own cache and return scalar scores and trust statistics. We do not evaluate that deployment here. Each candidate forward pass processes the cached prefix, with the language-model head evaluated only at selected positions. Timing and forward-pass counts are given in Appendix~\ref{app:efficiency}.

%% file: appendices/experimental_protocol.tex
\clearpage
\section{Experimental Setup}
\label{app:protocol}
This appendix specifies the data partition, training procedure, and evaluation protocol used in Section~\ref{sec:experiments}.

\subsection{Models and Client Data}
We distill Qwen3-8B into Qwen3-0.6B-Base and Qwen3-1.7B-Base, using the 0.6B model for the main study \citep{yang2025qwen3}. Client prompts are drawn from MATH \citep{hendrycks2021math}, using three disjoint subject groups: C1 contains prealgebra and algebra; C2 contains number theory and counting and probability; C3 contains intermediate algebra and precalculus. Each client receives 570 prompts. The partition is fixed with seed 42 and shared across methods. Client weights are proportional to their prompt counts, giving $w_i=1/3$. Although the prompts remain fixed, students generate new responses during training. We use the official training split of \texttt{EleutherAI/hendrycks\_math}, partitioned into three subject-based clients, with 570 OPD prompts per client selected without replacement using proportional subject--difficulty stratification (seed 42), disjoint from the reserved warm-up and probe subsets.
Preprocessing removes malformed examples, examples without an extractable and classifiable boxed answer, and duplicate questions across splits; during training, chat-formatted prompts exceeding 2,048 tokens are filtered rather than truncated, and each client draws batches of 32 prompts from a persistent shuffled data loader, generating one response per prompt with temperature $1$, top-$p=1$, and a 4,096-token response limit. Details can be found in the anonymous codes.

\subsection{Training and Baselines}
Within each student size, all trained methods optimize rank-8 LoRA adapters from the same initialization with AdamW (Table~\ref{tab:config}). We denote $\eta=3\times10^{-6}$ by 1$\times$LR and $9\times10^{-6}$ by 3$\times$LR. \method\ uses Algorithm~\ref{alg:fedtops}, with trust budget 0.05 for all student models.

\input{tables/training_config}

\paragraph{Training budgets.}
Each Local model uses 1,600 rollouts and 200 AdamW mini-batch steps. FedAvg and \method\ use the same per-client budget: 4,800 rollouts and 600 client steps over 50 rounds system-wide. Centralized uses 4,800 rollouts and 600 consecutive steps on pooled prompts. These totals match across methods, while sequential update depth differs. Token counts and training time depend on response length (Appendix~\ref{app:efficiency}).

\paragraph{Reference models.}
Local mean averages the three independently trained client-model scores. Base is an untrained, adapter-free Qwen3-0.6B-Base model, evaluated once and repeated in both learning-rate blocks. The 1.7B study compares Centralized, Local, FedAvg, and \method.

\subsection{Evaluation}
\paragraph{Benchmarks.}
We evaluate on six mathematical reasoning benchmarks, totaling 1,546 problems and 12,368 sampled responses per model (Table~\ref{tab:eval_sets}). Client-private test sets are excluded. All methods use the same evaluation prompt format, sampling protocol, and answer grader.

\input{tables/eval_sets}

\paragraph{Evaluation prompt.}
Evaluation is zero-shot. Each problem is supplied as a single user message, with the following instruction appended:
\begin{quote}
\small\ttfamily\raggedright
Please reason step by step, and put your final answer within \textbackslash boxed\{\}.
\end{quote}
\paragraph{Chat template and sampling.}
No separate system message is supplied to the tokenizer. We render the message with its chat template using \texttt{add\_generation\_prompt=True} and \texttt{enable\_thinking=False}. Each problem receives eight responses sampled with temperature 1, top-$p=0.8$, top-$k=-1$, minimum probability 0, seed 0, and a maximum response length of 8,192 tokens. Answers are extracted with \texttt{\seqsplit{qwen_extraction_math_verify_0p8_latex_v2}} and checked with \texttt{math-verify} 0.8.0. 

\paragraph{Metrics.}
Let $c_{bjm}\in\{0,1\}$ indicate whether response $m$ to problem $j$ in benchmark $b$ is correct. For a benchmark with $n_b$ problems,
\begin{equation}
\mathrm{Avg@8}_b=\frac{100}{8n_b}\sum_{j=1}^{n_b}\sum_{m=1}^{8}c_{bjm},\qquad
\mathrm{Pass@8}_b=\frac{100}{n_b}\sum_{j=1}^{n_b}\ind\!\left\{\sum_{m=1}^{8}c_{bjm}>0\right\}.
\end{equation}
Avg@8 measures average response correctness, whereas Pass@8 measures the fraction of problems solved at least once. Cap-hit rate is the percentage of responses reaching the generation length limit. Macro scores average the six benchmark scores with equal weights, including for cap-hit rate. Aggregation precedes rounding, and all reported scores are percentages.

\paragraph{Statistical units.}
Each configuration is trained once with seed 42. The eight responses per problem measure sampling variability within an evaluation, and Local mean averages client models rather than independent training repetitions. For the paired source intervention, bootstrap resampling uses the same question indices for all three evaluated models and retains all eight responses to each sampled question. These intervals describe evaluation uncertainty conditional on the trained models.

\subsection{Resources and Software}
Each experiment runs on one RTX PRO 6000 Blackwell Server GPU with 96 GB memory. The three clients are simulated on a single machine; additional GPUs are used for independent experiments. Table~\ref{tab:resource_versions} identifies the model and training-data versions. The grading dependencies are \texttt{latex2sympy2\_extended} 1.10.2, \texttt{sympy} 1.14.0, and ANTLR runtime 4.9.3. Our implementation builds on verl; the audited environment uses PyTorch 2.8.0 (CUDA 12.8), Transformers 4.57.1, PEFT 0.20.0, FlashAttention 2.8.3, and vLLM 0.11.0 for rollout generation and evaluation.

\begin{table}[htbp]
\centering\footnotesize
\caption{Model and data revisions for training. Exact identifiers pin the resources independently of later repository updates.}
\label{tab:resource_versions}
\setlength{\tabcolsep}{5pt}
\begin{adjustbox}{max width=\linewidth}
\begin{tabular}{ll}
\toprule

\rowcolor{headbg}\textbf{Resource} & \textbf{Revision}\\
\midrule
Qwen3-0.6B-Base & \texttt{da87bfb608c14b7cf20ba1ce41287e8de496c0cd}\\
Qwen3-1.7B-Base & \texttt{ea980cb0a6c2ae4b936e82123acc929f1cec04c1}\\
Qwen3-8B & \texttt{b968826d9c46dd6066d109eabc6255188de91218}\\
EleutherAI/hendrycks\_math & \texttt{7811dc4cf37bd3a0384adc84f33da0695629b4dc}\\
\bottomrule
\end{tabular}
\end{adjustbox}
\end{table}

The Base reference uses the same sampling protocol and prompt template, with different engine batching. Its checkpoint revision was not recorded. Generation settings and available artifact identifiers are provided in the accompanying reproducibility materials.

%% file: tables/training_config.tex
\begin{table}[htbp]
\centering
\caption{Training configuration. The local learning rate $\eta$ controls client optimization; the server multiplier $a$ scales the aggregated increment.}
\label{tab:config}
\small
\begin{tabularx}{\linewidth}{p{.32\linewidth}Y}
\toprule

\rowcolor{headbg}\textbf{Setting} & \textbf{Value}\\
\midrule
Participation / rounds & All 3 clients / 50 rounds\\
Rollouts per round & 32 per client, 1 response per prompt\\
Mini-batch / micro-batch & 8 / 4; one optimization epoch\\
Local learning rate & 0.6B: $3\times10^{-6}$ or $9\times10^{-6}$; 1.7B: $3\times10^{-6}$\\
Optimizer & AdamW, $\beta=(0.9,0.999)$, weight decay 0.01, gradient clipping 1\\
Scheduler & Cosine; no warmup\\
LoRA & Rank 8, $\alpha=16$, dropout 0, all-linear target modules\\
Aggregation & Equal-weight averaging in shared LoRA factor coordinates\\
Training decoding & Non-thinking, temperature 1, top-$p=1$\\
Training length limits & Prompt 2,048 tokens; response 4,096 tokens\\
Data / training seed & 42 / 42\\
\bottomrule
\end{tabularx}
\end{table}

%% file: tables/eval_sets.tex
\begin{table}[htbp]
\centering
\caption{Evaluation benchmarks. Each problem receives eight sampled responses.}
\label{tab:eval_sets}
\small
\begin{adjustbox}{max width=\linewidth}
\begin{tabular}{lrr}
\toprule

\rowcolor{headbg}\textbf{Benchmark} & \textbf{Problems} & \textbf{Responses}\\
\midrule
MATH500 & 500 & 4,000\\
AMC23 & 40 & 320\\
Minerva & 272 & 2,176\\
OlympiadBench & 674 & 5,392\\
AIME24 & 30 & 240\\
AIME25 & 30 & 240\\
\midrule
Total & 1,546 & 12,368\\
\bottomrule
\end{tabular}
\end{adjustbox}
\end{table}

%% file: appendices/additional_experiments.tex
\clearpage
\section{Additional Experimental Analyses}
\label{app:additional}
We examine the effect of aggregate scaling, the behavior of the selector, and the influence of rollout sources on subsequent learning. We then characterize computational cost and response length under the common evaluation protocol.

\subsection{Selection Across Communication Rounds}
\label{app:selector_config}
Table~\ref{tab:selector_counts} complements the round-level traces in Figure~\ref{fig:selection}. Candidate and cache settings are given in Table~\ref{tab:config}; Appendix~\ref{app:algorithm} specifies scoring and tie handling.

\input{tables/selector_counts}

At 1$\times$LR, all candidates are feasible throughout training, yet the selected multiplier changes across rounds. Thus teacher-score comparisons alone can favor different scales. At 3$\times$LR, 25 of the 300 candidates are infeasible across 18 rounds, and $a=10$ is selected less often. The traces characterize adaptation within each setting.

Figure~\ref{fig:selection} displays each candidate's teacher score relative to $a=1$ in the same round. Both panels use a shared symmetric-logarithmic color scale, linear within $\pm0.002$ nats, without clipping or smoothing. This makes the selected scales and feasibility boundaries comparable over training.

\subsection{Rollout-Source Intervention}
\label{app:source_protocol}
The intervention in Section~\ref{sec:mechanism_experiments} uses a common round-9 receiver checkpoint and frozen sources constructed from the same round-10 aggregate increment. We evaluate the receiving students on MATH500. Their cap-hit rates are 51.85\% with source $a=3$ and 53.30\% with source $a=1$, compared with 52.38\% at the common start.

\paragraph{Paired evaluation.}
Table~\ref{tab:source_detail} reports Avg@8 differences and paired 95\% bootstrap intervals. We resample the 500 questions 5,000 times, retaining each question's eight responses together and using the same question indices for all three models. These intervals characterize evaluation uncertainty for the saved models, not variation across training seeds. Differences use unrounded statistics and may differ from subtraction of the displayed endpoint scores.

\input{tables/source_detail}

\subsection{Computational Cost}
\label{app:efficiency}
The auxiliary forwards operate on at most 12 cached responses and 192 prediction positions per round. Table~\ref{tab:efficiency} reports caching and finalization time separately.

\input{tables/efficiency}

\paragraph{Forward computation.}
Cache construction adds 12 student and 12 teacher prefix forwards. Scoring a reference and six candidates adds 84 student forwards: 96 student and 12 teacher forwards in total per round. Teacher distributions are reused across candidates. Each pass processes the actual prefix, with only the LM head restricted to selected positions; cost therefore depends on prefix length.

\paragraph{Measured time and training tokens.}
Finalization includes scoring and saving; timing shares use the training-progress interval, excluding initialization and evaluation. Equal rollout counts need not imply equal training time because response lengths vary. The response-mask token totals are 6,445,612 for 0.6B at 1$\times$LR, 5,164,606 for 0.6B at 3$\times$LR, and 3,748,317 for 1.7B. These counts cover all 150 client-round batches and exclude prompt tokens and auxiliary forwards. The single-machine experiments characterize computation; they do not measure network traffic or comparative peak memory.

\input{appendices/response_lengths}

%% file: tables/selector_counts.tex
\begin{table}[htbp]
\centering
\caption{Candidate selection over 50 rounds with the 0.6B student. The trust budget is fixed within each run.}
\label{tab:selector_counts}
\small
\begin{adjustbox}{max width=\linewidth}
\begin{tabular}{lrr}
\toprule

\rowcolor{headbg}\textbf{Quantity} & \textbf{1$\times$LR} & \textbf{3$\times$LR}\\
\midrule
Trust budget $\delta$ & 0.05 & 0.05\\
Selected $a=0.5$ & 2 & 3\\
Selected $a=1$ & 3 & 9\\
Selected $a=2$ & 1 & 7\\
Selected $a=3$ & 3 & 8\\
Selected $a=5$ & 4 & 13\\
Selected $a=10$ & 37 & 10\\
Infeasible candidates (of 300) & 0 & 25\\
Rounds with infeasible candidates & 0 & 18\\
Holds & 0 & 0\\
\bottomrule
\end{tabular}
\end{adjustbox}
\end{table}

%% file: tables/source_detail.tex
\begin{table}[htbp]
\centering\small
\caption{Paired differences in receiver Avg@8 on MATH500 (pp).}
\label{tab:source_detail}
\begin{adjustbox}{max width=\linewidth}
\begin{tabular}{lrr}
\toprule

\rowcolor{headbg}\textbf{Comparison} & \textbf{Difference} & \textbf{95\% paired interval}\\
\midrule
Source $a=3$ receiver $-$ common start & 1.20 & [0.27, 2.10]\\
Source $a=1$ receiver $-$ common start & 0.20 & [$-$0.70, 1.07]\\
Source $a=3$ receiver $-$ source $a=1$ receiver & 1.00 & [0.18, 1.82]\\
\bottomrule
\end{tabular}
\end{adjustbox}
\end{table}

%% file: tables/efficiency.tex
\begin{table}[htbp]
\centering
\caption{Auxiliary computation for \method. Cache construction and candidate scoring/saving are timed over the training loop. Share is their sum divided by the training-loop interval; initialization and evaluation are excluded.}
\label{tab:efficiency}
\small\setlength{\tabcolsep}{4pt}
\begin{adjustbox}{max width=\linewidth}
\begin{tabular}{lrrrrr}
\toprule

\rowcolor{headbg}\textbf{Student / LR} & \textbf{Cache (s)} & \textbf{Score/save (s)} & \textbf{Sum (s)} & \textbf{Training (s)} & \textbf{Share (\%)}\\
\midrule
0.6B / 1$\times$ & 188.69 & 979.36 & 1168.05 & 13837 & 8.44\\
0.6B / 3$\times$ & 180.56 & 764.11 & 944.67 & 6227 & 15.17\\
1.7B / 1$\times$ & 171.22 & 754.80 & 926.03 & 7598 & 12.19\\
\bottomrule
\end{tabular}
\end{adjustbox}
\end{table}

%% file: appendices/response_lengths.tex
\subsection{Response Length and Generation Limits}
\label{app:lengths}
Table~\ref{tab:response_length} summarizes generation length. \method\ reduces cap hits at the reference learning rate for both student sizes, while the 0.6B high-rate run has a similar macro cap-hit rate to FedAvg despite higher accuracy. Cap hits characterize output behavior, not supervision quality.

\input{tables/response_lengths}

%% file: tables/response_lengths.tex
\begin{table}[htbp]
\centering\small
\caption{Mean response tokens / cap-hit rate (\%), equally averaged across benchmarks. The evaluation cap is 8,192 tokens. The 0.6B Base result is repeated for reference.}
\label{tab:response_length}
\begin{adjustbox}{max width=\linewidth}
\begin{tabular}{lrrr}
\toprule

\rowcolor{headbg}\textbf{Method} & \textbf{0.6B, 1$\times$LR} & \textbf{0.6B, 3$\times$LR} & \textbf{1.7B, 1$\times$LR}\\
\midrule
Base student & 4366.95 / 50.27 & 4366.95 / 50.27 & \textemdash\\
Centralized & 3675.20 / 38.18 & 4535.74 / 48.62 & 1537.92 / 8.98\\
Local mean & 4807.97 / 54.58 & 4120.63 / 44.23 & 3664.13 / 38.94\\
FedAvg & 4821.79 / 54.88 & 3164.08 / 31.16 & 3644.79 / 38.66\\
\rowcolor{oursbg}\textbf{FedTOPS} & 3830.62 / 38.31 & 3334.40 / 31.23 & 1617.09 / 9.89\\
\bottomrule\end{tabular}
\end{adjustbox}\end{table}

%% file: appendices/main_results.tex
\clearpage
\section{Client-Level Accuracy and Generation Statistics}
\label{app:mainresults}
The following tables complement the main accuracy comparisons with individual Local models and benchmark-level cap-hit rates. Local C1, C2, and C3 correspond to the subject partitions in Appendix~\ref{app:protocol}; Local mean averages their scores. Avg@8 and Pass@8 tables contain only these client results, while cap-hit tables compare all original training methods and the available Base reference.

All scores are percentages; Macro weights benchmarks equally. Best and second-best values among the rows in each table are bold and underlined. MATH, AMC, and Olympiad denote MATH500, AMC23, and OlympiadBench.

\subsection{0.6B Student at the Reference Learning Rate}
\label{app:full06_lr1}
\begin{table}[htbp]
\centering
\caption{0.6B student, 1$\times$LR: individual Local models, Avg@8 ($\uparrow$).}
\label{tab:full06_lr1_avg_at_8}
\small\setlength{\tabcolsep}{3pt}
\begin{adjustbox}{max width=\linewidth}
\begin{tabular}{lrrrrrrr}
\toprule\benchheader\midrule
\csname @@input\endcsname tables/main_06b_lr1_avg_at_8_rows.tex
\bottomrule
\end{tabular}
\end{adjustbox}
\end{table}
\begin{table}[htbp]
\centering
\caption{0.6B student, 1$\times$LR: individual Local models, Pass@8 ($\uparrow$).}
\label{tab:full06_lr1_pass_at_8}
\small\setlength{\tabcolsep}{3pt}
\begin{adjustbox}{max width=\linewidth}
\begin{tabular}{lrrrrrrr}
\toprule\benchheader\midrule
\csname @@input\endcsname tables/main_06b_lr1_pass_at_8_rows.tex
\bottomrule
\end{tabular}
\end{adjustbox}
\end{table}
\begin{table}[htbp]
\centering
\caption{0.6B student, 1$\times$LR: cap-hit rate ($\downarrow$).}
\label{tab:full06_lr1_max_length_hit_rate}
\small\setlength{\tabcolsep}{3pt}
\begin{adjustbox}{max width=\linewidth}
\begin{tabular}{lrrrrrrr}
\toprule\benchheader\midrule
\csname @@input\endcsname tables/main_06b_lr1_max_length_hit_rate_rows.tex
\bottomrule
\end{tabular}
\end{adjustbox}
\end{table}

\clearpage
\subsection{0.6B Student at the Larger Learning Rate}
\label{app:full06_lr3}
\begin{table}[htbp]
\centering
\caption{0.6B student, 3$\times$LR: individual Local models, Avg@8 ($\uparrow$).}
\label{tab:full06_lr3_avg_at_8}
\small\setlength{\tabcolsep}{3pt}
\begin{adjustbox}{max width=\linewidth}
\begin{tabular}{lrrrrrrr}
\toprule\benchheader\midrule
\csname @@input\endcsname tables/main_06b_lr3_avg_at_8_rows.tex
\bottomrule
\end{tabular}
\end{adjustbox}
\end{table}
\begin{table}[htbp]
\centering
\caption{0.6B student, 3$\times$LR: individual Local models, Pass@8 ($\uparrow$).}
\label{tab:full06_lr3_pass_at_8}
\small\setlength{\tabcolsep}{3pt}
\begin{adjustbox}{max width=\linewidth}
\begin{tabular}{lrrrrrrr}
\toprule\benchheader\midrule
\csname @@input\endcsname tables/main_06b_lr3_pass_at_8_rows.tex
\bottomrule
\end{tabular}
\end{adjustbox}
\end{table}
\begin{table}[htbp]
\centering
\caption{0.6B student, 3$\times$LR: cap-hit rate ($\downarrow$).}
\label{tab:full06_lr3_max_length_hit_rate}
\small\setlength{\tabcolsep}{3pt}
\begin{adjustbox}{max width=\linewidth}
\begin{tabular}{lrrrrrrr}
\toprule\benchheader\midrule
\csname @@input\endcsname tables/main_06b_lr3_max_length_hit_rate_rows.tex
\bottomrule
\end{tabular}
\end{adjustbox}
\end{table}

%% file: appendices/scaleup_results.tex
\clearpage
\subsection{1.7B Student at the Reference Learning Rate}
\label{app:scaleup}

\begin{table}[htbp]
\centering
\caption{Qwen3-1.7B-Base student after distillation: individual Local models, Avg@8 ($\uparrow$).}
\label{tab:full17_avg_at_8}
\small\setlength{\tabcolsep}{3pt}
\begin{adjustbox}{max width=\linewidth}
\begin{tabular}{lrrrrrrr}
\toprule\benchheader\midrule
\csname @@input\endcsname tables/scaleup_17b_avg_at_8_rows.tex
\bottomrule
\end{tabular}
\end{adjustbox}
\end{table}
\begin{table}[htbp]
\centering
\caption{Qwen3-1.7B-Base student after distillation: individual Local models, Pass@8 ($\uparrow$).}
\label{tab:full17_pass_at_8}
\small\setlength{\tabcolsep}{3pt}
\begin{adjustbox}{max width=\linewidth}
\begin{tabular}{lrrrrrrr}
\toprule\benchheader\midrule
\csname @@input\endcsname tables/scaleup_17b_pass_at_8_rows.tex
\bottomrule
\end{tabular}
\end{adjustbox}
\end{table}
\begin{table}[htbp]
\centering
\caption{Qwen3-1.7B-Base student after distillation: cap-hit rate ($\downarrow$).}
\label{tab:full17_max_length_hit_rate}
\small\setlength{\tabcolsep}{3pt}
\begin{adjustbox}{max width=\linewidth}
\begin{tabular}{lrrrrrrr}
\toprule\benchheader\midrule
\csname @@input\endcsname tables/scaleup_17b_max_length_hit_rate_rows.tex
\bottomrule
\end{tabular}
\end{adjustbox}
\end{table}

%% file: preprint.bbl
\begin{thebibliography}{23}
\providecommand{\natexlab}[1]{#1}
\providecommand{\url}[1]{\texttt{#1}}
\expandafter\ifx\csname urlstyle\endcsname\relax
  \providecommand{\doi}[1]{doi: #1}\else
  \providecommand{\doi}{doi: \begingroup \urlstyle{rm}\Url}\fi

\bibitem[Agarwal et~al.(2024)Agarwal, Vieillard, Zhou, Stanczyk, Ramos, Geist, and Bachem]{agarwal2024gkd}
Rishabh Agarwal, Nino Vieillard, Yongchao Zhou, Piotr Stanczyk, Sabela Ramos, Matthieu Geist, and Olivier Bachem.
\newblock On-policy distillation of language models: Learning from self-generated mistakes.
\newblock In \emph{International Conference on Learning Representations}, 2024.

\bibitem[Cai et~al.(2026)Cai, Cao, Lin, Luo, Xu, Yang, Liu, Yang, Zhao, Sun, Liu, and Fang]{cai2026foresee}
Yuchen Cai, Ding Cao, Liang Lin, Chunxi Luo, Xin Xu, Kai Yang, Weijie Liu, Saiyong Yang, Tianxiang Zhao, Guangzhong Sun, Guiquan Liu, and Junfeng Fang.
\newblock Learning to foresee: Unveiling the unlocking efficiency of on-policy distillation.
\newblock \emph{arXiv preprint arXiv:2605.11739}, 2026.

\bibitem[Fu et~al.(2026)Fu, He, Zuo, Huang, Zhang, Xiao, Qian, Luo, Gao, Wang, Liu, Ding, and Xiao]{fu2026oneexample}
Zixuan Fu, Bingxiang He, Yuxin Zuo, Haohuan Huang, Jinqian Zhang, Ruhang Xiao, Cheng Qian, Qinyu Luo, Huan-ang Gao, Yudong Wang, Zhiyuan Liu, Ning Ding, and Chaojun Xiao.
\newblock Rethinking on-policy distillation of large language models {II}: One training example.
\newblock \emph{arXiv preprint arXiv:2609.04172}, 2026.

\bibitem[Gu et~al.(2024)Gu, Dong, Wei, and Huang]{gu2024minillm}
Yuxian Gu, Li~Dong, Furu Wei, and Minlie Huang.
\newblock {MiniLLM}: Knowledge distillation of large language models.
\newblock In \emph{International Conference on Learning Representations}, 2024.

\bibitem[He et~al.(2024)He, Luo, Bai, Hu, Thai, Shen, Hu, Han, Huang, Zhang, Liu, Qi, Liu, and Sun]{he2024olympiadbench}
Chaoqun He, Renjie Luo, Yuzhuo Bai, Shengding Hu, Zhen~Leng Thai, Junhao Shen, Jinyi Hu, Xu~Han, Yujie Huang, Yuxiang Zhang, Jie Liu, Lei Qi, Zhiyuan Liu, and Maosong Sun.
\newblock {OlympiadBench}: A challenging benchmark for promoting {AGI} with olympiad-level bilingual multimodal scientific problems.
\newblock In \emph{Proceedings of the 62nd Annual Meeting of the Association for Computational Linguistics}, 2024.

\bibitem[Hendrycks et~al.(2021)Hendrycks, Burns, Kadavath, Arora, Basart, Tang, Song, and Steinhardt]{hendrycks2021math}
Dan Hendrycks, Collin Burns, Saurav Kadavath, Akul Arora, Steven Basart, Eric Tang, Dawn Song, and Jacob Steinhardt.
\newblock Measuring mathematical problem solving with the {MATH} dataset.
\newblock In \emph{NeurIPS Datasets and Benchmarks}, 2021.

\bibitem[Hou et~al.(2026)Hou, Zhang, Cai, and You]{hou2026what}
Zhinan Hou, Jiaqi Zhang, Xunliang Cai, and Keyou You.
\newblock What matters in on-policy distillation? a perspective on data efficiency and data selection.
\newblock \emph{arXiv preprint arXiv:2609.05198}, 2026.

\bibitem[Jhunjhunwala et~al.(2023)Jhunjhunwala, Wang, and Joshi]{jhunjhunwala2023fedexp}
Divyansh Jhunjhunwala, Shiqiang Wang, and Gauri Joshi.
\newblock {FedExP}: Speeding up federated averaging via extrapolation.
\newblock In \emph{International Conference on Learning Representations}, 2023.

\bibitem[Jin et~al.(2024)Jin, Yin, Chen, Sun, Zhang, Liu, and Liu]{jin2024performative}
Kun Jin, Tongxin Yin, Zhongzhu Chen, Zeyu Sun, Xueru Zhang, Yang Liu, and Mingyan Liu.
\newblock Performative federated learning: A solution to model-dependent and heterogeneous distribution shifts.
\newblock In \emph{Proceedings of the AAAI Conference on Artificial Intelligence}, volume~38, pp.\  12938--12946, 2024.

\bibitem[Jin et~al.(2026)Jin, Min, Yang, Wei, Zhou, Kadhe, Baracaldo, and Lee]{jin2026entropy}
Woogyeol Jin, Taywon Min, Yongjin Yang, Dennis Wei, Yi~Zhou, Swanand~Ravindra Kadhe, Nathalie Baracaldo, and Kimin Lee.
\newblock Entropy-aware on-policy distillation of language models.
\newblock \emph{arXiv preprint arXiv:2603.07079}, 2026.

\bibitem[Karimireddy et~al.(2020)Karimireddy, Kale, Mohri, Reddi, Stich, and Suresh]{karimireddy2020scaffold}
Sai~Praneeth Karimireddy, Satyen Kale, Mehryar Mohri, Sashank Reddi, Sebastian Stich, and Ananda~Theertha Suresh.
\newblock {SCAFFOLD}: Stochastic controlled averaging for federated learning.
\newblock In \emph{Proceedings of the 37th International Conference on Machine Learning}, volume 119 of \emph{Proceedings of Machine Learning Research}, pp.\  5132--5143, 2020.

\bibitem[Lewkowycz et~al.(2022)Lewkowycz, Andreassen, Dohan, Dyer, Michalewski, Ramasesh, Slone, Anil, Schlag, Gutman-Solo, Wu, Neyshabur, Gur-Ari, and Misra]{lewkowycz2022minerva}
Aitor Lewkowycz, Anders Andreassen, David Dohan, Ethan Dyer, Henryk Michalewski, Vinay Ramasesh, Ambrose Slone, Cem Anil, Imanol Schlag, Theo Gutman-Solo, Yuhuai Wu, Behnam Neyshabur, Guy Gur-Ari, and Vedant Misra.
\newblock Solving quantitative reasoning problems with language models.
\newblock In \emph{Advances in Neural Information Processing Systems}, volume~35, 2022.

\bibitem[Li et~al.(2026)Li, Zuo, He, Zhang, Xiao, Qian, Yu, Gao, Yang, Liu, and Ding]{li2026rethinking}
Yaxuan Li, Yuxin Zuo, Bingxiang He, Jinqian Zhang, Chaojun Xiao, Cheng Qian, Tianyu Yu, Huan-ang Gao, Wenkai Yang, Zhiyuan Liu, and Ning Ding.
\newblock Rethinking on-policy distillation of large language models: Phenomenology, mechanism, and recipe.
\newblock \emph{arXiv preprint arXiv:2604.13016}, 2026.

\bibitem[Mangold et~al.(2025)Mangold, Berthier, and Moulines]{mangold2025fedsarsa}
Paul Mangold, Elo{\"i}se Berthier, and Eric Moulines.
\newblock Convergence guarantees for federated {SARSA} with local training and heterogeneous agents.
\newblock \emph{arXiv preprint arXiv:2512.17688}, 2025.

\bibitem[McMahan et~al.(2017)McMahan, Moore, Ramage, Hampson, and Arcas]{mcmahan2017fedavg}
Brendan McMahan, Eider Moore, Daniel Ramage, Seth Hampson, and Blaise Aguera~y Arcas.
\newblock Communication-efficient learning of deep networks from decentralized data.
\newblock In \emph{Proceedings of the 20th International Conference on Artificial Intelligence and Statistics}, volume~54 of \emph{Proceedings of Machine Learning Research}, pp.\  1273--1282, 2017.

\bibitem[Perdomo et~al.(2020)Perdomo, Zrnic, Mendler-D{\"u}nner, and Hardt]{perdomo2020performative}
Juan Perdomo, Tijana Zrnic, Celestine Mendler-D{\"u}nner, and Moritz Hardt.
\newblock Performative prediction.
\newblock In \emph{Proceedings of the 37th International Conference on Machine Learning}, volume 119 of \emph{Proceedings of Machine Learning Research}, pp.\  7599--7609, 2020.

\bibitem[Ross et~al.(2011)Ross, Gordon, and Bagnell]{ross2011dagger}
St{\'e}phane Ross, Geoffrey Gordon, and Drew Bagnell.
\newblock A reduction of imitation learning and structured prediction to no-regret online learning.
\newblock In \emph{Proceedings of the Fourteenth International Conference on Artificial Intelligence and Statistics}, volume~15 of \emph{Proceedings of Machine Learning Research}, pp.\  627--635, 2011.

\bibitem[Schulman et~al.(2015)Schulman, Levine, Abbeel, Jordan, and Moritz]{schulman2015trpo}
John Schulman, Sergey Levine, Pieter Abbeel, Michael Jordan, and Philipp Moritz.
\newblock Trust region policy optimization.
\newblock In \emph{Proceedings of the 32nd International Conference on Machine Learning}, volume~37 of \emph{Proceedings of Machine Learning Research}, pp.\  1889--1897, 2015.

\bibitem[Singhal et~al.(2025)Singhal, Ponkshe, and Vepakomma]{singhal2025fedexlora}
Raghav Singhal, Kaustubh Ponkshe, and Praneeth Vepakomma.
\newblock {FedEx-LoRA}: Exact aggregation for federated and efficient fine-tuning of large language models.
\newblock In \emph{Proceedings of the 63rd Annual Meeting of the Association for Computational Linguistics (Volume 1: Long Papers)}, pp.\  1316--1336, 2025.

\bibitem[Sun et~al.(2024)Sun, Li, Li, and Ding]{sun2024ffa}
Youbang Sun, Zitao Li, Yaliang Li, and Bolin Ding.
\newblock Improving {LoRA} in privacy-preserving federated learning.
\newblock In \emph{International Conference on Learning Representations}, 2024.

\bibitem[Wang et~al.(2024)Wang, Shen, He, Sun, Wang, Lyu, and Li]{wang2024flora}
Ziyao Wang, Zheyu Shen, Yexiao He, Guoheng Sun, Hongyi Wang, Lingjuan Lyu, and Ang Li.
\newblock {FLoRA}: Federated fine-tuning large language models with heterogeneous low-rank adaptations.
\newblock In \emph{Advances in Neural Information Processing Systems}, volume~37, 2024.

\bibitem[Yang et~al.(2025)Yang, Li, Yang, Zhang, Hui, Zheng, Yu, Gao, Huang, Lv, et~al.]{yang2025qwen3}
An~Yang, Anfeng Li, Baosong Yang, Beichen Zhang, Binyuan Hui, Bo~Zheng, Bowen Yu, Chang Gao, Chengen Huang, Chenxu Lv, et~al.
\newblock {Qwen3} technical report.
\newblock \emph{arXiv preprint arXiv:2505.09388}, 2025.

\bibitem[Zhang et~al.(2024)Zhang, Wang, Mitra, and Anderson]{zhang2024fedsarsa}
Chenyu Zhang, Han Wang, Aritra Mitra, and James Anderson.
\newblock Finite-time analysis of on-policy heterogeneous federated reinforcement learning.
\newblock In \emph{International Conference on Learning Representations}, 2024.

\end{thebibliography}
